\pdfoutput=1
\documentclass{article}
\usepackage{iclr2027_conference,times}
\iclrfinalcopy

\usepackage{float}
\usepackage{hyperref}
\usepackage{url}

\usepackage{amsmath}
\usepackage{amssymb}
\usepackage{bm}

\usepackage{amsmath,amsfonts,bm}

\def\eqref#1{equation~\ref{#1}}

\def\1{\bm{1}}

\DeclareMathAlphabet{\mathsfit}{\encodingdefault}{\sfdefault}{m}{sl}
\SetMathAlphabet{\mathsfit}{bold}{\encodingdefault}{\sfdefault}{bx}{n}

\usepackage{booktabs}
\newcommand{\tablefont}{\scriptsize}
\usepackage[section]{placeins}
\usepackage{needspace}
\usepackage{multirow}
\usepackage{graphicx}
\usepackage{xcolor}
\usepackage{colortbl}
\definecolor{glmTint}{RGB}{225,237,250}
\definecolor{dsTint}{RGB}{253,241,225}
\definecolor{qwTint}{RGB}{224,246,233}

\newcommand{\modelcell}[4]{{\setlength{\fboxsep}{0.7pt}\colorbox{glmTint}{\makebox[#1][r]{\strut #2}}\hspace{0.12em}\colorbox{dsTint}{\makebox[#1][r]{\strut #3}}\hspace{0.12em}\colorbox{qwTint}{\makebox[#1][r]{\strut #4}}}}
\providecommand{\tablefont}{\scriptsize}
\newcommand{\estci}[3]{\begin{tabular}[c]{@{}r@{}}\strut #1\\[-0.2ex]\ensuremath{#2}\\[-0.2ex]\ensuremath{#3}\end{tabular}}
\newcommand{\statpoint}[1]{\strut #1}
\newcommand{\countci}[3]{\begin{tabular}[c]{@{}r@{}}\strut #1\\[-0.2ex]\ensuremath{[#2,}\\[-0.2ex]\ensuremath{#3]}\end{tabular}}

\usepackage{microtype}

\usepackage[capitalize,noabbrev]{cleveref}
\crefname{appendix}{Appendix}{Appendices}
\Crefname{appendix}{Appendix}{Appendices}

\usepackage{amsthm}
\theoremstyle{plain}
\newtheorem{theorem}{Theorem}[section]

\theoremstyle{definition}

\theoremstyle{remark}

\usepackage{enumitem}
\setlist[itemize]{leftmargin=*}
\setlist[enumerate]{leftmargin=*}

\usepackage{algorithm}
\usepackage{algpseudocode}

\hypersetup{
  pdftitle={When to Compile a Computer-Use Agent? Measuring Payback and Making Compilation Decisions for Token Efficiency},
  pdfauthor={Yulong Ming, Jie Xu, Zihan Wu, Xiaohua Jia}}

\newcommand{\codeurl}{https://github.com/myl7/computer-use}

\title{When to Compile a Computer-Use Agent?\\
Measuring Payback and Making Compilation Decisions for Token Efficiency}

\author{Yulong Ming, Jie Xu, Zihan Wu, Xiaohua Jia\\
City University of Hong Kong\\
\texttt{myl.7@my.cityu.edu.hk}}

\begin{document}

\maketitle
\lhead{Preprint}

\begin{abstract}
\looseness=-1
Compiling GUI procedures that agents execute repeatedly into programs can reduce their token costs.
However, measuring payback and deciding when to compile have two challenges.
First, compilation costs are uncertain because attempts can require repair and still fail to produce a usable program.
Second, future reuse is unknown because tasks may stop arriving or GUI drift may stop the program from working.
To address these challenges, we propose PACE (Payback-Aware Compilation from Experience), a system with a measurement protocol and an online compilation algorithm.
The measurement protocol records successful and failed compilation costs, and compares agent and program execution costs on matched task inputs to estimate per-use savings and payback counts.
Using these measurements, the online algorithm compares estimated future savings with compilation costs, including failed attempts, based on past task arrivals and compilation outcomes.
It checks execution and compilation charges against a cumulative budget determined by observed task arrivals before allowing either action.
Under stated action-cost assumptions, total cost after each arrival is at most $1+\epsilon$ times the cost of running every task with the agent.
For successful compilation attempts, estimated payback counts excluding source agent runs are $2$--$16$ uses.
In simulations using recorded task arrivals, PACE reduces token costs by $17.3\%$ compared with ReAct, $24.9\%$ with the AutoRPA adaptation, and $17.3\%$ with the ToolPro adaptation on average ($\epsilon=0.25$).
\end{abstract}

\section{Introduction}
\label{sec:intro}

Large language models have advanced rapidly.
This progress enables computer-use agents.
A computer-use agent interacts with applications through graphical user interfaces (GUIs).
At each step, the agent reads the GUI as both a screenshot and a structured text tree, e.g., an accessibility or DOM tree \citep{hong2024cogagent,zhang2025ufo}.
It then runs a ReAct-style \citep{yao2023react} agent loop, i.e., it calls the model to generate actions and executes them on the interface.
Task success rates keep rising on benchmarks of web, desktop, and mobile tasks \citep{zhou2024webarena,xie2024osworld,rawles2025androidworld}.
However, computer-use agents consume more tokens than general agents because of the screenshot inputs, making them more expensive to run.
Token efficiency, i.e., reducing the token cost, has become a first-order concern.

To improve token efficiency, prior systems reuse earlier experience.
When an agent solves similar tasks repeatedly, starting from scratch requires it to rediscover the same procedure each time.
To reduce this cost, prior systems store earlier experience as text (e.g., memory items or skills) \citep{wang2025awm,zheng2024synapse} or \textit{compile} it into parameterized programs by generating code \citep{wang2025asi,li2026preact,chen2026autorpa,elyadouni2026tracecompiler}.
A model call extracts the task inputs, which the program then uses to execute the procedure without further model calls.
We focus on the program method, which has the highest token efficiency.

However, a low execution cost does not mean that compilation saves tokens over the full task stream.
Prior audits report payback counts ranging from 4 \citep{chen2026autorpa} to 17 \citep{trooskens2026compiledai} uses, where the payback count is the number of uses required for token savings to cover the compilation cost.
Most systems report no payback count \citep{elyadouni2026tracecompiler,xu2026skillsmith}.
Moreover, these audits measure the costs of different system parts, making the numbers incomparable.
We therefore ask two questions.
\emph{How can we systematically measure the payback count?}
\emph{When is it worth compiling a program instead of continuing to use the agent?}

Measuring the payback count and deciding when to compile have two challenges.
\textbf{Uncertain compilation cost.}
Producing a program may require multiple compilation attempts.
Each attempt may consume tokens on repair, i.e., model calls that modify the program so it can complete the task.
If the attempt produces no usable program, the system pays the compilation cost and continues to use the agent.
\textbf{Unknown future use.}
The system does not know how many later tasks can reuse the program.
Similar tasks may stop arriving before they pay back the spend.
GUI changes that make a deployed program stop working, i.e., \emph{GUI drift}, can require a new compilation before further reuse.
Therefore, compiling early can waste tokens if tasks do not repeat, while waiting can lose savings on tasks that repeat often.
An online algorithm is required to make this choice using past observations.

\begin{figure}[htb]
\vspace{-2pt}%
\centering
\includegraphics[width=\linewidth]{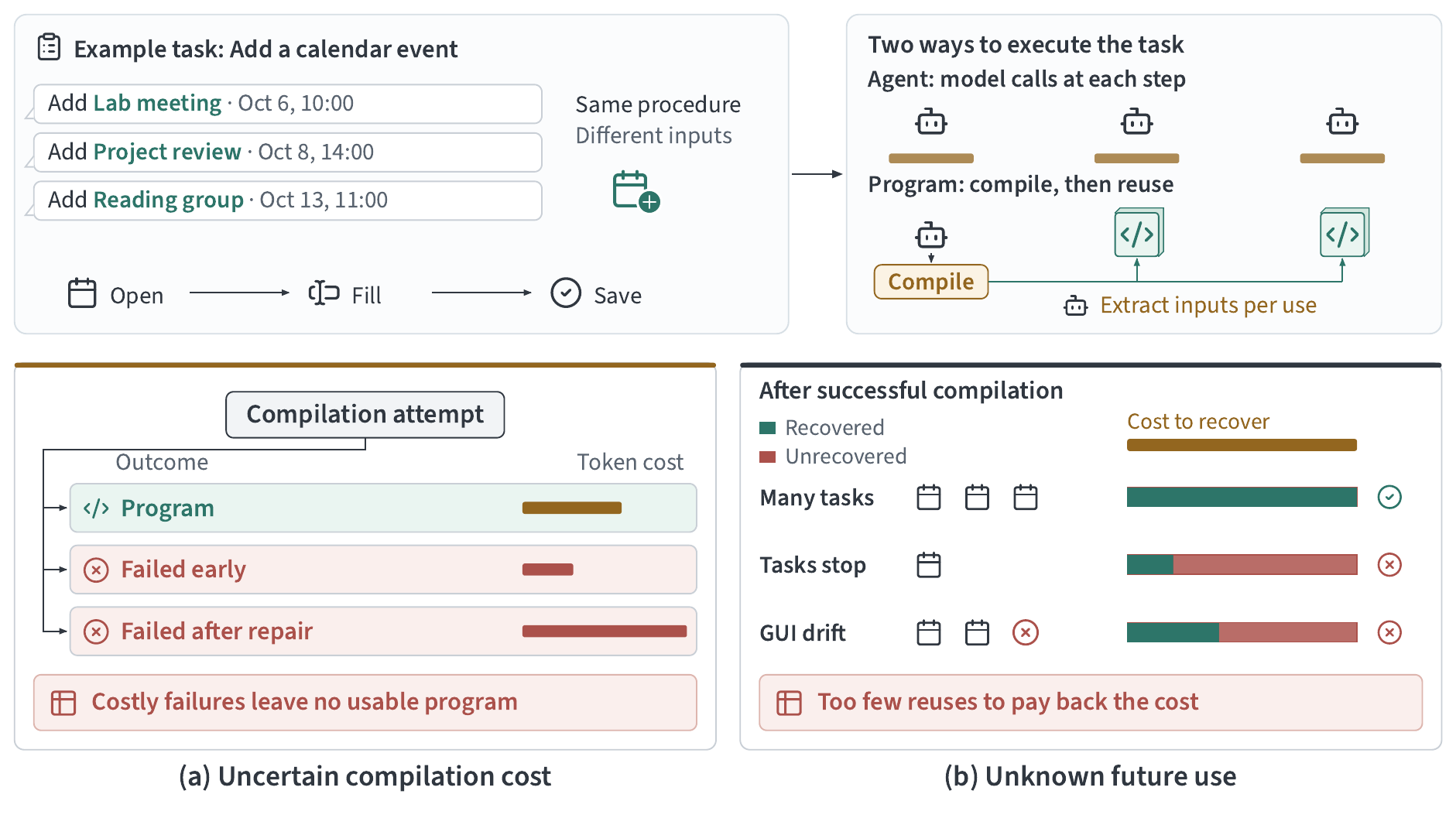}
\caption{\looseness=-1 Challenges in deciding when to compile GUI procedures.
Top: an example task and its agent and program execution paths.
In (a), red cost bars mark failed compilation attempts.
In (b), red bar segments show compilation costs left unrecovered when task arrivals stop or GUI drift ends reuse.}
\label{fig:teaser}
\vspace{-3pt}%
\end{figure}

We propose \emph{PACE} (Payback-Aware Compilation from Experience), a system with a measurement protocol and an online compilation algorithm.
PACE uses experience in two ways: completed agent traces provide the procedure to compile, while past task arrivals and compilation outcomes inform the decision to compile.
The measurement separates the cost of producing a program from the saving after that program exists.
We compare agent and program execution costs on matched task inputs, and estimate the additional cost of agent fallback after a detected program failure.
The protocol compiles recorded agent traces, and records the token costs of successful and failed attempts.
Deployment runs measure extraction cost and program failures.
Additional compilations test whether the initial outcome repeats, and agent replays on the program's task inputs check this cost comparison.
Across seven task families and three models, the measurements show per-use savings alongside differences in compilation cost and outcome.

PACE uses these measurements to decide when to compile under a cumulative cost constraint.
A cost-aware compilation proposal compares estimated future savings with the estimated compilation cost of producing a program, including failed attempts.
A separate budget check allows an action only when its possible cost fits the budget from task arrivals already observed.
Total cost includes task execution, routing, and every successful or failed compilation.
Under the stated cost model, PACE spends at most $1+\epsilon$ times the reference cost of running every task with the agent after each arrival.
The margin $\epsilon$ limits the extra spend even when future-reuse estimates are wrong or compilation keeps failing.
We evaluate the resulting savings and tradeoffs with measured cost profiles, synthetic arrivals, and arrival orders from public logs.

\Needspace{4\baselineskip}
Our contributions are:
\begin{itemize}
\item We propose a measurement protocol for the costs and savings of compiling computer-use tasks.
We report task-execution costs, compilation outcomes, repair costs, and payback counts under a shared accounting scheme, with repeated compilations and matched task-input replays.
\item We propose an online compilation algorithm that uses a cost budget to decide whether to compile.
PACE makes cost-aware proposals from past experience and checks their costs against a cumulative budget.
We prove a cumulative bound under action-cost assumptions.
\item We implement PACE and evaluate it under different cost settings and task arrival patterns.
We release the measurement records, policy implementations, and paired experimental outcomes.
With $\epsilon=0.25$, PACE reduces token costs by $17.3\%$ relative to ReAct, $24.9\%$ relative to the AutoRPA adaptation, and $17.3\%$ relative to the ToolPro adaptation on average across the recorded-arrival simulations.
Ablations separate the effects of the savings test, cost budget, and failed-attempt cost.
\end{itemize}

\section{Problem Formulation}
\label{sec:formulation}

We study when to compile GUI procedures that agents execute repeatedly while limiting cumulative cost relative to continued agent execution.
Tasks in one family share a GUI procedure but differ in their inputs, e.g., event titles and dates.

We consider two ways to execute, or \emph{serve}, a task.
In \emph{agent execution}, the model selects GUI actions step by step, with mean cost $c$.
In \emph{program execution}, a model call extracts the parameter values from the task prompt, i.e., a \emph{binding}.
The program then executes the shared procedure with these values, without further model calls.
Its model cost $c^{\mathrm{prog}}$ is the extraction cost.
We use weighted tokens: one unit corresponds to one uncached input token at the model's reference price, with cached-input and output tokens weighted by their relative prices.
Appendix~\ref{app:accounting} explains how we calculate these costs.

A compilation attempt uses $k$ completed agent traces.
The attempt charge includes all model calls made during compilation.
Let $p$ be the probability that compilation succeeds, $C$ the mean cost of a successful attempt, and $C^{\mathrm{fail}}$ the mean cost of a failed attempt.
A failed attempt still consumes tokens.
In the online model, these cost values become fixed charges for the corresponding outcomes.
Historical means alone do not bound future actual costs.

\looseness=-1
On each program execution, GUI drift makes the program stop working with probability $h$.
If there is no GUI drift, execution fails with probability $q$.
Either event requires one agent fallback of cost $c$, but only GUI drift stops further program reuse.
The total fallback probability is $h+(1-h)q$, reducing to $q$ when $h=0$.
The cost model assumes detected failures and reports task success separately.
Routing selects an available program or the agent for the current task.
It costs $\tau_0+m\ell_t$ when $\ell_t$ programs are listed, where $\tau_0$ is the cost with an empty list and $m$ is the cost per entry.

\paragraph{Objective.}
We compare PACE with serving every arrival through agent execution, without compiling or reusing programs.
For arrival $t$, define this agent execution reference charge as $a_t=\tau_0+c$ using that family's given cost, and let
\begin{equation}
A_t=\sum_{i=1}^{t}a_i.
\label{eq:agent-reference}
\end{equation}
This reference includes program-served arrivals.
It is a model-based comparison, not an extra agent execution observed on each task.
Let $K_t^\pi$ include policy $\pi$'s routing, serving, fallback, and all compilation charges through arrival $t$, including compilation after the final service.
For an unknown stream length $n$, we seek
\begin{equation}
\min_{\pi\in\Pi_{\mathrm{online}}}\mathbb{E}[K_n^\pi]
\quad\text{subject to}\quad
K_t^\pi\leq(1+\epsilon)A_t,\quad t=1,\ldots,n.
\label{eq:objective}
\end{equation}
The constraint holds for every history allowed by the cost model, while the expectation uses the arrival and outcome model evaluated in the experiments.
The set $\Pi_{\mathrm{online}}$ contains policies that serve each arrival through an agent or an available program, using only the current task, given parameters, and past observations.
We give the policy family labels, cost profiles, routing charges, and $q,h$.
The policy estimates future arrivals of each task family and the probability of successful compilation online.
Looking up a program or removing its list entry costs no tokens.
Savings can fund compilation across families, but future tasks that have not arrived add no budget.
The margin is $\epsilon=0.25$ in the main evaluation.
This objective constrains cost, without imposing a task-success guarantee.

\section{Measurement}
\label{sec:measure}

We measure both the saving after a program exists and the cost of producing it.
In each attempt, we run the generated program on its source tasks and check whether it completes them.
We call this check \emph{verification}.
We record the costs of successful and failed attempts, and compare program and agent execution on matched fresh inputs.
These measurements provide the online cost profiles and show why low program-use cost alone does not establish payback.

\begin{figure}[H]
\vspace{-5pt}%
\centering
\includegraphics[width=\linewidth]{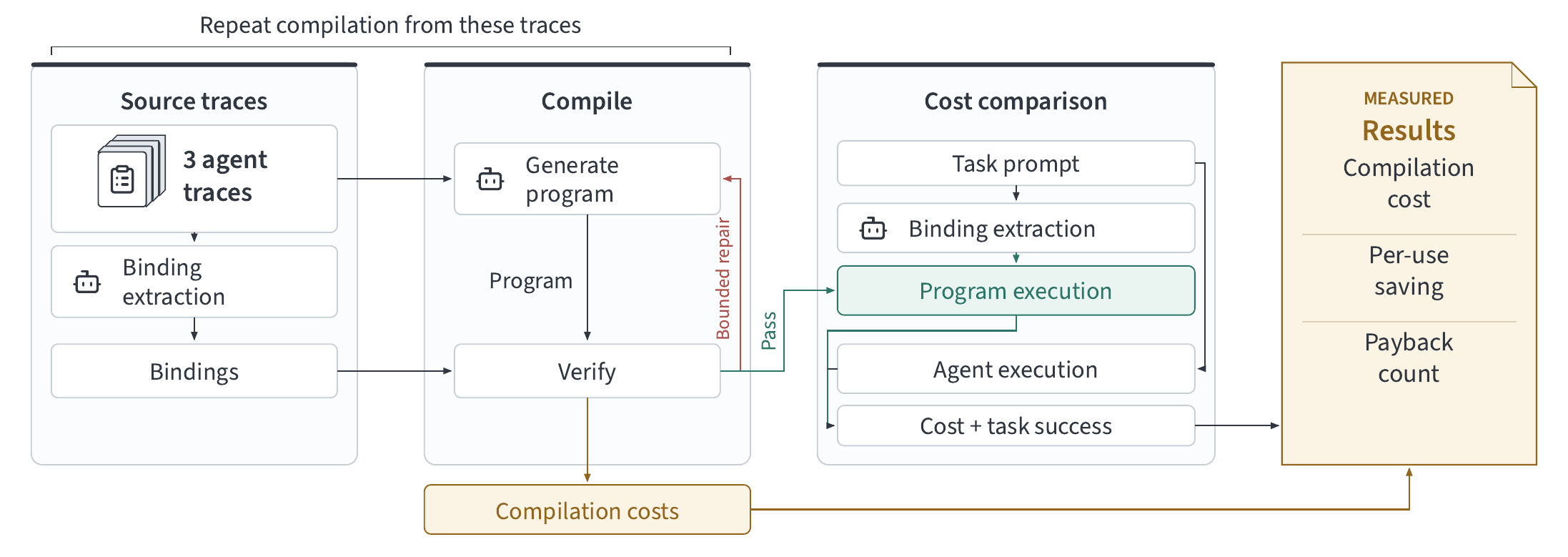}
\caption{PACE's measurement protocol.
The source traces provide compilation experience, and their task prompts provide the model's inputs for binding extraction during verification.
Fresh task prompts are used to compare program and agent execution.
The result panel combines execution measurements with the costs of successful and failed compilation attempts.}
\label{fig:workflow}
\vspace{-7pt}%
\end{figure}

\looseness=-1
We measure compilation and serving costs on seven families: four AndroidWorld tasks \citep{rawles2025androidworld}, two custom LibreOffice tasks in an OSWorld-based container \citep{xie2024osworld}, and a comment-posting task on WebArena's Reddit application \citep{zhou2024webarena}.
The models are GLM 5.3 Flash, DeepSeek Flash, and Qwen 3.8 Flash, with exact recorded identifiers in the Appendix.
Following AutoRPA's translation and construction stages \citep{chen2026autorpa}, we compile $k=3$ completed agent traces and permit bounded repair.
We assume access to task resets and success checks for these source tasks.
For verification, the model extracts bindings from the three source task prompts.
We cache these bindings within the attempt, restore each task's initial state, and require all three tasks to succeed before deployment.
We pass the extracted bindings to the program and judge its results against the source task goals.
This check measures reproduction of the source tasks.
Thirty fresh deployment tasks provide a separate evaluation, including model-based binding extraction.
We subtract the measured no-task harness cost from agent runs, and keep every recorded compilation and repair charge.
The reported costs exclude calls without saved records.
Appendix~\ref{app:measurement-details} describes the verification steps, cost calculations, repeated compilations, and which records are available.

\subsection{Serving cost}
\label{sec:share}

For a program, setting $h=0$ gives the expected saving per execution
\begin{equation}
s=(1-q)c-c^{\mathrm{prog}},\qquad\text{saving share}=s/c.
\label{eq:saving}
\end{equation}
\looseness=-1
The saving calculation includes extraction cost and one mean-cost fallback per failed program run.
Eleven initial programs pass verification.
Their extraction costs range from $331$ to $724$ tokens and their estimated saving shares from $0.936$ to $0.999$.
Appendix~\ref{app:serving-profiles} reports the profiles used in the experiments.

\label{sec:paired}
\looseness=-1
We use agent replays on the thirty deployment bindings of each of the six verified Android programs, giving $180$ matched-input agent runs (Table~\ref{tab:paired-replays} in Appendix~\ref{app:serving-profiles}).
Their saving shares range from $0.952$ to $0.999$, with $157$ successful agent runs and $178$ successful program runs out of $180$.
E.g., the GLM Calendar program passes all thirty bindings, while the replayed agent passes fourteen.
The replay uses a different emulator profile from the program runs and estimates fallback cost as $qc$; it does not measure an agent resumed after each failed program run.

\subsection{Compilation cost}
\label{sec:compile-cost}

Table~\ref{tab:compile-cost} reports the initial attempt charges and verification on the three source tasks.
Using the recorded successful charge as $C$, the payback count excluding source runs is $C/s$.
Adding $k$ mean-cost source runs gives $(C+kc)/s$.
These ratios exclude earlier failed attempts, extra runs required to collect traces, and replacement after drift.
We report cost ratios without rounding up to whole uses.
Rounding up gives the number of uses required to cover the recorded charge in expectation.
Programs from eleven of the twenty cells with compilation-stage records pass verification.
The other nine attempts fail, including Web DeepSeek's generation failure with an incomplete charge record.
Qwen's refused web request remains outside these twenty cells.

\begin{table}[htbp]
\vspace{-10pt}%
\centering
\caption{Compilation costs, verification outcomes, and conditional payback.
The two payback columns report $C/s$ and $(C+kc)/s$, with $k=3$ agent runs here.
Payback excludes earlier failed attempts and future replacement.
Blue, amber, and green identify GLM, DeepSeek, and Qwen, respectively.
Verification entries report passed/failed/untested task counts in that order; G marks a generation failure, and $\geq$ marks a lower-bound cost.}
\label{tab:compile-cost}
\tablefont
\setlength{\tabcolsep}{1.8pt}
\renewcommand{\arraystretch}{0.95}
\begin{tabular}{@{}lrrrrr@{}}
\toprule
\multicolumn{1}{c}{Family} & \multicolumn{1}{c}{\shortstack{Attempt cost\\(K tokens)}} & \multicolumn{1}{c}{\shortstack{Before repair\\(3/3 required)}} & \multicolumn{1}{c}{\shortstack{After repair\\(3/3 required)}} & \multicolumn{1}{c}{\shortstack{Compilation-only\\payback (uses)}} & \multicolumn{1}{c}{\shortstack{Trace-inclusive\\payback (uses)}} \\
\midrule
Contacts & \modelcell{3.2em}{137}{605}{138} & \modelcell{2.4em}{0/1/2}{0/1/2}{3/0/0} & \modelcell{2.4em}{3/0/0}{0/1/2}{3/0/0} & \modelcell{2.8em}{1.39}{--}{3.7} & \modelcell{2.8em}{4.4}{--}{6.85} \\
Markor delete & \modelcell{3.2em}{144}{156}{212} & \modelcell{2.4em}{0/1/2}{0/1/2}{3/0/0} & \modelcell{2.4em}{3/0/0}{3/0/0}{3/0/0} & \modelcell{2.8em}{4.31}{15.5}{10} & \modelcell{2.8em}{7.34}{18.7}{13.1} \\
Calendar & \modelcell{3.2em}{3424}{2872}{1936} & \modelcell{2.4em}{0/1/2}{0/1/2}{0/1/2} & \modelcell{2.4em}{3/0/0}{0/1/2}{0/1/2} & \modelcell{2.8em}{10.9}{--}{--} & \modelcell{2.8em}{13.9}{--}{--} \\
OsmAnd marker & \modelcell{3.2em}{3082}{1701}{1552} & \modelcell{2.4em}{0/1/2}{0/1/2}{0/1/2} & \modelcell{2.4em}{2/1/0}{0/1/2}{0/1/2} & \modelcell{2.8em}{--}{--}{--} & \modelcell{2.8em}{--}{--}{--} \\
\addlinespace[2pt]
Calc table & \modelcell{3.2em}{2544}{3057}{3054} & \modelcell{2.4em}{0/1/2}{0/1/2}{0/1/2} & \modelcell{2.4em}{0/1/2}{3/0/0}{0/1/2} & \modelcell{2.8em}{--}{4.73}{--} & \modelcell{2.8em}{--}{7.74}{--} \\
Writer memo & \modelcell{3.2em}{473}{259}{487} & \modelcell{2.4em}{0/1/2}{0/1/2}{3/0/0} & \modelcell{2.4em}{3/0/0}{3/0/0}{3/0/0} & \modelcell{2.8em}{6.8}{5.18}{13.4} & \modelcell{2.8em}{9.82}{8.22}{16.5} \\
\addlinespace[2pt]
Reddit comment & \modelcell{3.2em}{437}{$\geq$1397}{--} & \modelcell{2.4em}{0/1/2}{G}{--} & \modelcell{2.4em}{3/0/0}{G}{--} & \modelcell{2.8em}{3.55}{--}{--} & \modelcell{2.8em}{6.57}{--}{--} \\
\bottomrule
\end{tabular}
\vspace{-8pt}%
\end{table}

Among DeepSeek attempts with complete charge records, failed attempts have median cost $1701$K tokens, compared with $259$K for successful attempts, a $6.57$-fold difference across families.
Qwen's Calc attempt costs $3054$K tokens and still fails the first source task after three refinements.
This comparison does not identify the within-family cost effect of failure.
After successful compilation, the payback count excluding source runs ranges from $1.389$ to $15.488$ uses across these programs.
Thus low execution cost can still require many reuses to pay back compilation.

\section{Online Algorithm}
\label{sec:policy}

\looseness=-1
PACE's online compilation algorithm uses observed task arrivals and compilation outcomes to make cost-aware compilation proposals, with a separate cumulative budget to limit spending.
The proposal can be wrong about future arrivals or verification, while the budget uses only arrivals already observed and costs already charged.
The algorithm serves the current task before considering compilation for later tasks.

\begin{figure}[tbp]
\vspace{-5pt}%
\centering
\includegraphics[width=\linewidth]{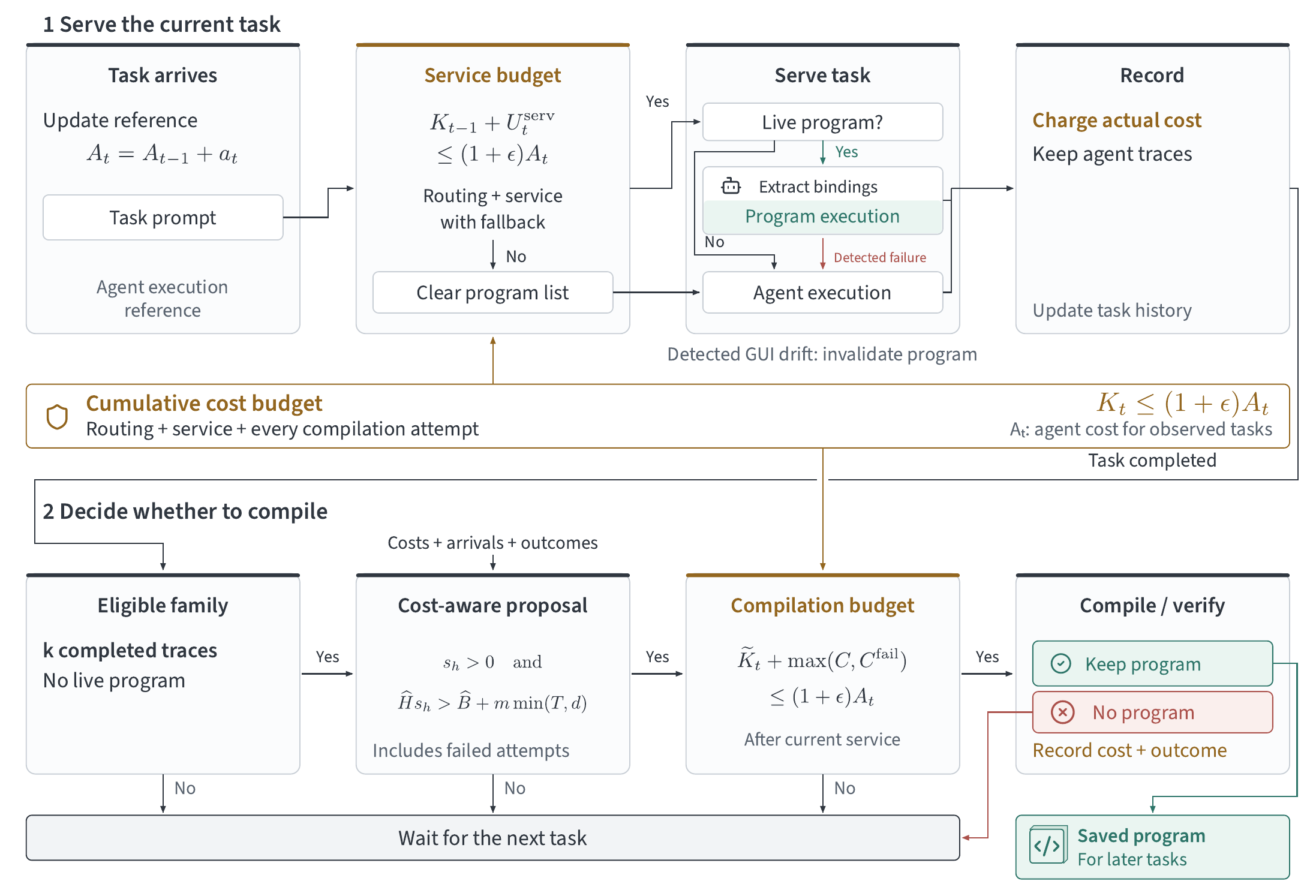}
\caption{\looseness=-1 PACE's online compilation algorithm.
After serving a task, PACE checks trace availability, estimated savings, and the cost budget before compiling.
A failed check leads to waiting, while successful verification adds a program for later tasks.
The two budget checks use the same cumulative cost budget.}
\label{fig:online-algorithm}
\vspace{-6pt}%
\end{figure}

\subsection{Cost-aware compilation proposal}
\label{sec:fullpolicy}

A family has enough traces for compilation after $k$ completed agent runs, including fallbacks.
Let $N$ be the number of arrivals observed for this family.
Let $d=t-t_{\mathrm{first}}$, where $t_{\mathrm{first}}$ is the index of its first arrival in the full task stream.
Thus $d=0$ on the family's first arrival.
After $j$ compilation attempts and $g$ successful verifications, define
\begin{equation}
\hat H=\min\!\left(\frac{Nd}{d+5},\frac1h\right),\qquad
\hat p=\frac{g+1}{j+2},\qquad
\hat B=C+\left(\frac1{\hat p}-1\right)C^{\mathrm{fail}}.
\label{eq:estimated-compile-cost}
\end{equation}
\looseness=-1
We use $1/0=\infty$.
This count estimates future reuse, with the mean program lifetime defined by $h$ as its upper limit.
We compute it from observed arrivals without training a prediction model.
The estimated compilation cost $\hat B$ includes failed attempts before verification succeeds, under the independent-retry interpretation derived in Appendix~\ref{app:reference}.
The cost guarantee does not require an accurate reuse estimate or independent retry outcomes.

With drift-adjusted saving $s_h=c-c^{\mathrm{prog}}-[h+(1-h)q]c$, a family with enough traces and no program to reuse proposes compilation when
\begin{equation}
s_h>0\quad\text{and}\quad\hat Hs_h>\hat B+m\min(T,d).
\label{eq:trigger}
\end{equation}
\looseness=-1
We call this condition the savings test.
The final term estimates the routing cost while one program remains on the list.
PACE removes entries after $T$ arrivals without use but keeps the programs.
It can add a program to the list again unless GUI drift has made it stop working.
The proposal accounts for the two sources of uncertain payback separately.
Low verification probability increases the estimated number of paid failures before producing a program, whereas GUI drift reduces the estimated reuse count after verification.
An expensive failed attempt can therefore make compilation unattractive even when successful construction is cheap.
The following budget checks limit the cost when these estimates are wrong.
Appendix~\ref{sec:matched-budget} compares alternative proposals with the same budget.
Their similar costs do not identify the exact estimators in Equation~\ref{eq:estimated-compile-cost} as a separate contribution.

\subsection{Budget checks}
\label{sec:budget-checks}

Before routing, PACE removes idle entries and adds any saved program that can still be reused for the arriving family to the list.
The next budget check decides whether this list can be used.
The program list shows the model which programs are available for routing.
This routing cost is paid on every arrival, including tasks served by the agent.
Keeping a program costs no tokens, but leaving it on the list adds routing cost on later tasks.
For a list of $\ell_t$ programs, PACE reserves enough budget for the largest possible service charge:
\begin{equation}
U_t^{\mathrm{serv}}=\tau_0+m\ell_t+
\begin{cases}
c,&\text{agent service},\\
c^{\mathrm{prog}}+c,&\text{program service}.
\end{cases}
\label{eq:service-reserve}
\end{equation}
The program reservation includes a possible fallback.
Service proceeds only if $K_{t-1}+U_t^{\mathrm{serv}}\leq(1+\epsilon)A_t$.
Otherwise PACE clears the program list before routing and uses the agent, whose charge is at most $a_t$.
Clearing the list keeps the programs and does not require any future spending.

Let $\widetilde K_t$ be the cumulative charge after service.
A proposed compilation proceeds only if
\begin{equation}
\widetilde K_t+\max(C,C^{\mathrm{fail}})\leq(1+\epsilon)A_t.
\label{eq:compile-budget}
\end{equation}
PACE then adds the actual compilation charge to its cumulative cost, including failed attempts and any compilation after the last service.
Savings can fund compilation across families, but estimated future reuse adds no budget.
Appendix~\ref{app:protocol-details} gives the complete execution order.

\subsection{Cost guarantee}
\label{sec:guarantee}

\begin{theorem}[Cost bound relative to agent execution]
\label{thm:cost-bound}
Let the reference increments $a_t$ be nonnegative.
Suppose every proposed action has a known finite charge bound, and clearing the program list before routing costs no tokens, requires no future spending, and allows default service at cost at most $a_t$.
For any $\epsilon\geq0$, PACE with exact arithmetic and valid reservations satisfies
\begin{equation}
K_t\leq(1+\epsilon)A_t
\quad\text{after each arrival }t
\label{eq:cost-guarantee}
\end{equation}
on every history allowed by these charge bounds.
\end{theorem}

\begin{proof}
The bound holds after both steps.
If service passes the budget check, its cost is at most the reserved amount.
If the check fails, default service adds at most $a_t$, so
\[
K_{t-1}+a_t\leq(1+\epsilon)A_{t-1}+a_t\leq(1+\epsilon)A_t.
\]
The compilation check keeps the same bound for either outcome.
Starting from $K_0=A_0=0$ establishes the result by induction.
\end{proof}

\looseness=-1
The fixed-charge model satisfies the assumptions with $a_t=\tau_0+c$.
The implementation permits tolerance $10^{-8}\max\{1,(1+\epsilon)A_t\}$.
The bound is global over the stream and concerns cost, not task errors or per-family spending.
When actual costs vary, spending must stay within the reserved amounts, and default service must cost at most $a_t$.
Historical means alone do not ensure these conditions.
PACE satisfies the cost constraint in Equation~\ref{eq:objective}, but we do not claim that it minimizes expected cost.

\section{Online Evaluation}
\label{sec:eval}

We evaluate the savings available under PACE's cumulative budget, its comparison with published decision rules, and the contribution of each component.
We simulate task execution and compilation using measured cost profiles on synthetic streams and recorded arrival orders.

The evaluation uses two grids of $300$ conditions, each crossing three models, four verification settings, five arrival patterns, and five cost or environment settings, with eight paired runs per condition.
The second grid uses its own seeds and parameters.
We also use Sepsis and BPI 2019 process logs \citep{mannhardt2016sepsis,vandongen2019bpi} and two Wikipedia edit streams, giving twelve combinations of models and logs with ten paired runs.
The logs give task frequencies and arrival order, not live GUI tasks.
We shuffle measured cost profiles onto families and assign verification probabilities independently of initial compilation outcomes.
We estimate costs for successful or failed outcomes that were not observed and test how these estimates affect the results below.
The policies receive family labels, failure parameters, and an indication of detected GUI drift.
Appendix~\ref{app:simulation-details} specifies the grids, mappings, estimates for missing costs, and shared outcome sequences.

We compare PACE with ReAct \citep{yao2023react} and decision-rule adaptations of AutoRPA \citep{chen2026autorpa} and ToolPro \citep{liu2026toolpro} under the same compiler and cost model.
ReAct serves every arrival with the agent.
AutoRPA makes one compilation attempt when enough source traces are available and does not retry a rejected attempt.
ToolPro compares the compilation cost with the saving available from the current task, without projecting future reuse.
These adaptations differ only in when they call the same compiler.
All policies pay the same routing, extraction, fallback, and actual compilation charges, so differences in compiler quality do not affect this comparison of decision rules.
Each table ratio divides a policy's mean cost by PACE's mean cost in the same condition.
Overall results average these ratios with equal weights.
Reported percentage savings are the equally weighted means of $100(1-K_{\mathrm{PACE}}/K_{\mathrm{baseline}})$ across conditions, using each condition's mean costs.
We use $2000$ paired bootstrap resamples of runs, conditional on the tested conditions and given profiles.
We used these grids while developing PACE, so the comparisons are not independent confirmation.
The main tables give point estimates, with complete conditional intervals in Appendix~\ref{app:full-comparisons}.

\subsection{Savings}
\label{sec:savings}

Across the grids, ReAct costs $1.107$ and $1.130$ times PACE on average (Table~\ref{tab:policysim}).
Across $4800$ grid runs and $120$ log runs, every recorded cumulative cost ratio to the agent reference stays within the implemented $1.25$ bound.
This checks the implementation on those runs.
Theorem~\ref{thm:cost-bound} covers all histories allowed by the model.
AutoRPA's largest grid ratios to PACE are $19.314$ and $26.538$.
ToolPro matches ReAct cost in these conditions because its compilation rule counts savings from only one task.

\begin{table}[!htbp]
\vspace{-8pt}%
\centering
\caption{Cost ratios relative to PACE over two grids of 300 conditions each.
Mean is the mean of the ratios for individual conditions, and max is the largest tested condition mean.
Blue, amber, and green identify GLM, DeepSeek, and Qwen, respectively.
Bold marks the best displayed values.
Conditional 95\% intervals are in Appendix~\ref{app:full-comparisons}.}
\label{tab:policysim}
\tablefont
\setlength{\tabcolsep}{0.6pt}
\renewcommand{\arraystretch}{0.95}
\begin{tabular}{@{}lrrrr@{}}
\toprule
\multicolumn{1}{c}{Policy} & \multicolumn{1}{c}{\shortstack{Grid 1 mean\\($\times$ PACE) $\downarrow$}} & \multicolumn{1}{c}{\shortstack{Grid 1 max\\($\times$ PACE) $\downarrow$}} & \multicolumn{1}{c}{\shortstack{Grid 2 mean\\($\times$ PACE) $\downarrow$}} & \multicolumn{1}{c}{\shortstack{Grid 2 max\\($\times$ PACE) $\downarrow$}} \\
\midrule
ReAct & \modelcell{3.4em}{1.121}{1.144}{1.056} & \modelcell{3.4em}{2.796}{3.495}{2.281} & \modelcell{3.4em}{1.147}{1.174}{1.070} & \modelcell{3.4em}{3.692}{3.128}{2.742} \\
AutoRPA & \modelcell{3.4em}{2.107}{1.942}{2.715} & \modelcell{3.4em}{11.997}{9.660}{19.314} & \modelcell{3.4em}{2.168}{2.006}{2.805} & \modelcell{3.4em}{16.034}{12.927}{26.538} \\
ToolPro & \modelcell{3.4em}{1.121}{1.144}{1.056} & \modelcell{3.4em}{2.796}{3.495}{2.281} & \modelcell{3.4em}{1.147}{1.174}{1.070} & \modelcell{3.4em}{3.692}{3.128}{2.742} \\
\midrule
\textbf{PACE} & \modelcell{3.4em}{\textbf{1.000}}{\textbf{1.000}}{\textbf{1.000}} & \modelcell{3.4em}{\textbf{1.000}}{\textbf{1.000}}{\textbf{1.000}} & \modelcell{3.4em}{\textbf{1.000}}{\textbf{1.000}}{\textbf{1.000}} & \modelcell{3.4em}{\textbf{1.000}}{\textbf{1.000}}{\textbf{1.000}} \\
\bottomrule
\end{tabular}
\vspace{-6pt}%
\end{table}

\begin{table}[!htbp]
\vspace{-4pt}%
\centering
\caption{Cost ratios relative to PACE on four recorded arrival sequences, using ten paired runs.
Blue, amber, and green identify GLM, DeepSeek, and Qwen, respectively.
Bold marks the best displayed values.
Conditional 95\% intervals are in Appendix~\ref{app:full-comparisons}.}
\label{tab:realstreams}
\tablefont
\setlength{\tabcolsep}{0.6pt}
\renewcommand{\arraystretch}{0.95}
\begin{tabular}{@{}lrrrr@{}}
\toprule
\multicolumn{1}{c}{Policy} & \multicolumn{1}{c}{\shortstack{Sepsis\\($\times$ PACE) $\downarrow$}} & \multicolumn{1}{c}{\shortstack{BPI 2019\\($\times$ PACE) $\downarrow$}} & \multicolumn{1}{c}{\shortstack{Wiki tools\\($\times$ PACE) $\downarrow$}} & \multicolumn{1}{c}{\shortstack{Wiki humans\\($\times$ PACE) $\downarrow$}} \\
\midrule
ReAct & \modelcell{3.4em}{\textbf{0.971}}{\textbf{0.998}}{\textbf{0.989}} & \modelcell{3.4em}{1.531}{1.603}{1.393} & \modelcell{3.4em}{1.730}{1.480}{1.288} & \modelcell{3.4em}{1.056}{1.060}{1.040} \\
AutoRPA & \modelcell{3.4em}{1.145}{1.123}{1.246} & \modelcell{3.4em}{1.582}{1.641}{1.489} & \modelcell{3.4em}{1.754}{1.498}{1.331} & \modelcell{3.4em}{1.164}{1.144}{1.232} \\
ToolPro & \modelcell{3.4em}{\textbf{0.971}}{\textbf{0.998}}{\textbf{0.989}} & \modelcell{3.4em}{1.531}{1.603}{1.393} & \modelcell{3.4em}{1.730}{1.480}{1.288} & \modelcell{3.4em}{1.056}{1.060}{1.040} \\
\midrule
\textbf{PACE} & \modelcell{3.4em}{1.000}{1.000}{1.000} & \modelcell{3.4em}{\textbf{1.000}}{\textbf{1.000}}{\textbf{1.000}} & \modelcell{3.4em}{\textbf{1.000}}{\textbf{1.000}}{\textbf{1.000}} & \modelcell{3.4em}{\textbf{1.000}}{\textbf{1.000}}{\textbf{1.000}} \\
\bottomrule
\end{tabular}
\vspace{-8pt}%
\end{table}

\looseness=-1
On the recorded logs, PACE reduces token costs by $17.34\%$ relative to ReAct, with a conditional $95\%$ interval of $[14.96,19.61]\%$ (Table~\ref{tab:realstreams}).
It reduces cost in nine of twelve conditions.
Its largest cost ratio to ReAct is $1.030$, and its lowest is $0.578$.
Relative to PACE, AutoRPA and ToolPro have mean cost ratios of $1.362$ and $1.262$, respectively.
AutoRPA costs more in every condition, while ToolPro costs less in three.
The comparison concerns the stated decision-rule adaptations, rather than the original systems end to end.
Cost-aware proposals under the same budget have similar costs (Appendix~\ref{sec:matched-budget}).

\subsection{Component effects}
\label{sec:ablation}

\looseness=-1
Table~\ref{tab:ablation} reports cost ratios relative to full PACE.
\emph{No savings test} compiles whenever eligible while keeping the budget.
\emph{No cost budget} removes reservations and budget-triggered program list clearing.
\emph{Eager retry} removes both, and \emph{Equal attempt costs} uses $C/\hat p$ in place of $\hat B$ while keeping actual charges and reservations.

\begin{table}[htbp]
\centering
\caption{Cost ratios after removing or simplifying PACE components.
Values are mean cost ratios relative to PACE, which is exactly one.
Blue, amber, and green identify GLM, DeepSeek, and Qwen, respectively.
Paired intervals are in Appendix~\ref{app:ablation-absolute}.
Grid max is the largest tested condition-mean ratio, without a sampling interval.}
\label{tab:ablation}
\tablefont
\setlength{\tabcolsep}{0.9pt}
\renewcommand{\arraystretch}{0.95}
\begin{tabular}{@{}lrrrr@{}}
\toprule
\multicolumn{1}{c}{Variant} & \multicolumn{1}{c}{\shortstack{Grid 1 mean\\($\times$ PACE)}} & \multicolumn{1}{c}{\shortstack{Grid 2 mean\\($\times$ PACE)}} & \multicolumn{1}{c}{\shortstack{Logs mean\\($\times$ PACE)}} & \multicolumn{1}{c}{\shortstack{Grid max\\($\times$ PACE)}} \\
\midrule
\textbf{PACE} & \modelcell{3.4em}{1.000}{1.000}{1.000} & \modelcell{3.4em}{1.000}{1.000}{1.000} & \modelcell{3.4em}{1.000}{1.000}{1.000} & \modelcell{3.4em}{1.000}{1.000}{1.000} \\
No savings test & \modelcell{3.4em}{1.174}{1.186}{1.181} & \modelcell{3.4em}{1.140}{1.167}{1.121} & \modelcell{3.4em}{1.581}{1.548}{1.451} & \modelcell{3.4em}{1.957}{1.810}{1.901} \\
No cost budget & \modelcell{3.4em}{1.008}{1.014}{1.007} & \modelcell{3.4em}{1.015}{1.016}{1.010} & \modelcell{3.4em}{1.000}{1.000}{1.000} & \modelcell{3.4em}{1.371}{1.410}{1.301} \\
Equal attempt costs & \modelcell{3.4em}{1.100}{1.093}{1.119} & \modelcell{3.4em}{1.082}{1.095}{1.077} & \modelcell{3.4em}{1.227}{1.175}{1.297} & \modelcell{3.4em}{1.568}{1.567}{1.466} \\
Eager retry & \modelcell{3.4em}{10.317}{8.414}{18.329} & \modelcell{3.4em}{10.919}{9.352}{19.286} & \modelcell{3.4em}{3.303}{3.326}{4.922} & \modelcell{3.4em}{92.863}{73.996}{195.631} \\
\bottomrule
\end{tabular}
\end{table}

\looseness=-1
Removing the savings test raises mean cost on the recorded logs by $52.68\%$ (conditional $95\%$ interval $[48.25,58.36]\%$).
Using the same estimated cost for successful and failed attempts raises it by $23.34\%$ ($[19.18,28.06]\%$).
The equal-cost variant raises mean cost for all three models.
Both variants keep the budget, so their largest grid ratios to ReAct remain near the same $1.25$ limit despite higher mean costs.
Removing the budget changes mean cost on the recorded logs by only $0.0013\%$, yet raises the largest tested grid ratio to ReAct from $1.2468$ to $1.7563$, or $40.9\%$.
Thus the near-zero effect on the recorded logs concerns average cost, while the budget's measured contribution is protection against larger losses.
Appendix~\ref{app:ablation-absolute} gives the cost ratios and paired intervals, including conditions where an ablation is cheaper.

\paragraph{Cost assumptions.}
For missing costs, we vary the assumed ratio of failed to successful compilation costs over $0.5,1,2,5$.
This gives mean PACE-to-ReAct cost ratios from $0.833$ to $0.866$ on the recorded logs.
Assigning high compilation costs to frequent families with low verification probability raises the mean to $1.015$, about $1.5\%$ above ReAct, while preserving the cost constraint.
This assignment makes programs costly to produce for the families with the most reuse opportunities.
The resulting loss of savings shows why task frequency alone is not enough to decide when to compile.
The budget still limits extra spending, but it cannot create savings when the available compilation costs and useful lifetimes make reuse unprofitable.
Appendix~\ref{app:cost-sensitivity} reports all seventy-two sensitivity conditions.
Appendix~\ref{app:live} reports a separate live GUI component study.

\section{Related Work}
\label{sec:related}

\textbf{Computer-use agents.}
Synapse and Agent Workflow Memory reuse trajectories or extracted workflows to guide later GUI actions \citep{zheng2024synapse,wang2025awm}.
AutoRPA and PreAct compile interaction traces into programs, with verification, repair, or agent fallback \citep{chen2026autorpa,li2026preact}.
PACE measures successful and failed compilation costs, and decides when to compile under a cumulative token budget.

\looseness=-1
\textbf{Online decision theory.}
Classical rent-or-buy analysis weighs ongoing service against an upfront purchase \citep{karlin1994competitive}.
Prediction-guided algorithms improve purchase timing using estimated demand while limiting the effect of prediction errors \citep{purohit2018improving,gollapudi2019online}.
Conservative contextual bandits constrain cumulative reward relative to a baseline at every round \citep{kazerouni2016conservative}.
PACE applies this baseline constraint to token charges, reserving compilation, fallback, and routing costs under valid action-cost bounds.

\section{Limitations}
\label{sec:limitations}

The measurements cover seven parameterized task families and three model profiles.
The online simulations use given family labels, fixed cost profiles, and reported failure causes.
Their savings depend on how costs and compilation outcomes are assigned to arriving task families.
Historical mean costs do not provide the action bounds required for the cost guarantee.
The live study tests fallback and call reservations for a fixed program, without evaluating the full online compilation algorithm.
Silent program failures and recovery from partially changed interfaces require further evaluation.
We measure model-token cost and exclude latency, device use, and program-maintenance effort.
Model and provider updates can change the measured costs.

\section{Conclusion}
\label{sec:conclusion}

\looseness=-1
PACE measures successful and failed compilation costs and per-use savings, then uses these profiles to make cost-aware compilation proposals.
Execution and compilation checks use the budget from observed arrivals to enforce $K_t\leq(1+\epsilon)A_t$ under valid action-cost bounds.
For successful compilation attempts, estimated payback counts excluding source agent runs are $2$--$16$ uses.
In simulations using recorded task arrivals, PACE saves $17.3\%$ relative to ReAct, $24.9\%$ relative to the AutoRPA adaptation, and $17.3\%$ relative to the ToolPro adaptation ($\epsilon=0.25$).
Ablations separate the effects of the savings test, failed-attempt cost, and the cost budget.

\FloatBarrier
\label{main:text-end}
\providecommand{\codeurl}{https://anonymous.4open.science/r/computer-use}
\subsection*{AI use statement}
\label{sec:ai-use}
\looseness=-1
We used AI tools for finding relevant literature, drafting and revising text, and implementing the measurement and evaluation methods.
The tools also helped plan experiments, combine and interpret recorded results, develop and check mathematical arguments, and prepare tables and figures.
AI-generated programs are also part of the compilation method evaluated in Section~\ref{sec:measure}.
We have reviewed all AI-assisted work.
We take responsibility for the final content of this paper.

\subsection*{Ethics statement}
\label{sec:ethics}
The study uses benchmark applications, an Android emulator, desktop applications in a Linux container, and locally hosted web applications.
The arrival experiments use public event logs.
The study does not involve recruiting human participants or executing generated actions on users' personal accounts.
Our cost bound does not ensure task correctness, so applying the method to actions with important real-world effects would require checks appropriate to those actions.

\subsection*{Reproducibility statement}
\label{sec:reproducibility}
\looseness=-1
The code and saved records are at \expandafter\url\expandafter{\codeurl}.
Appendix~\ref{app:measurement-details} describes the task families, recorded model identifiers, cost calculations, and verification procedures.
Section~\ref{sec:guarantee} states and proves the cost bound, and Appendix~\ref{app:protocol-details} gives the complete execution order and reference calculations.
Appendix~\ref{app:simulation-details} specifies the experimental grids, arrival logs, estimates of missing costs, paired outcomes, and confidence interval calculation.
Appendix~\ref{app:live} describes the live GUI study and its charge records.
The released files include scripts for reading measurements, running experiments, and creating tables, along with paired results and lists of the inputs used.

\bibliography{references}
\bibliographystyle{iclr2027_conference}

\clearpage
\appendix
\raggedbottom
\crefalias{section}{appendix}
\crefalias{subsection}{appendix}
\section{Measurement details}\label{app:measurement-details}

\subsection{Task families}
\label{sec:families}

\looseness=-1
The study covers seven parameterized task families across Android, desktop, and web interfaces.
Four come from AndroidWorld \citep{rawles2025androidworld}: adding a contact, deleting a Markor note, adding a calendar event, and saving an OsmAnd marker.
The desktop tasks save a small Calc table or a Writer memo in a Linux container based on the OSWorld environment \citep{xie2024osworld}.
The web task posts a comment on WebArena's self-hosted Reddit application \citep{zhou2024webarena}.
The desktop and web tasks use custom parameterized procedures and task-specific state checks.
Task generators and checks are in \path{code/guiexp_android/android_env.py}, \path{code/guiexp_osworld/families.py}, and \path{code/guiexp_webarena/family.py} in the released files.
The released files list the result paths and hashes used in this paper.
These tasks let us measure costs on more interfaces, but they do not evaluate the full OSWorld or WebArena benchmarks.

\looseness=-1
The recorded model identifiers are \nolinkurl{z-ai/glm-5.3-flash}, \nolinkurl{deepseek/deepseek-v4-flash-vision-exp}, and \nolinkurl{qwen/qwen3.8-flash}.
The tables abbreviate these profiles as GLM, DeepSeek, and Qwen.
Each agent reads an unmarked screenshot and a numbered text tree at every step.
The tree is an accessibility tree on Android and desktop, and a DOM tree on the web.
The planned comparison has $21$ family--model cells.
Qwen's provider refuses the web request before producing a compilation record, leaving twenty cells for the source-task verification study.
Qwen Calc exhausts its repair budget without passing the source-task check and is recorded as a failed attempt.

\subsection{Cost accounting}\label{app:accounting}
\looseness=-1
Input, cached-input, and output tokens can have different prices.
For a call with $x$ input tokens, including $x_{\mathrm{cache}}$ cached tokens, and $y$ output tokens, its token cost is $(x-x_{\mathrm{cache}})+r_c x_{\mathrm{cache}}+r_o y$.
Here $r_c$ and $r_o$ are the cached-input and output prices relative to the uncached-input price.
One unit corresponds to one uncached input token at that model's recorded reference price.
We use $(r_c,r_o)=(1/5,10/3)$ for GLM, $(7/220,3)$ for DeepSeek, and $(8/75,47/15)$ for Qwen.
The conversion values used for these measurements are saved with the released measurement records.
We keep these ratios fixed across provider routes and record actual monetary charges separately.
We sum these units over serving, compilation, and repair calls.
The cost calculation excludes wall-clock time, device resources, and human effort.
Dollar comparisons across models also require their uncached-input prices.

\looseness=-1
We subtract the measured no-task harness baseline from agent runs.
We do not subtract it from compilation or binding extraction.
The Android baselines are $5090$ tokens for GLM, $2290$ for DeepSeek, and about $1032$ for Qwen.
Most inputs in Qwen's baseline measurements are cache hits, so this value is smaller than its raw-token baseline.
The desktop and web calculations use their own recorded baseline measurements.

We calculate costs from saved call records.
For every deployment, we calculate extraction cost from each recorded call's input, cached-input, and output token counts using the fixed ratios above.
Compilation costs include recorded retries and repair calls.
Calls lost during interrupted desktop runs cannot be recovered from these records.
The Web DeepSeek attempt described below lacks the usage records for three early empty builder replies, so its saved compilation charge is a lower bound.
One Qwen Contacts repeated attempt also has two timed-out calls without usage records, so its recorded charge is a lower bound.
We convert recorded token counts from official Qwen calls using the same fixed ratios.
Their separate monetary records are in yuan and are not divided by a dollar-denominated input price.

\subsection{Compilation}
\label{sec:protocol}

Three completed agent traces provide the source experience for each initial compilation attempt.
A translator rewrites their actions, and a builder constructs the parameterized program, following AutoRPA's translation and construction stages \citep{chen2026autorpa}.
The model extracts one binding from each source task prompt using the program's parameter schema.
We permit one retry when the extracted output fails the JSON or type check and charge both calls when a retry occurs.
We cache the three extracted bindings within the compilation attempt and use them for every refined program with the same parameter interface.
If extraction fails, we reject the attempt.
The verifier restores each source task's initial state.
A program can be used on later tasks only if it completes all three source tasks.
We pass the extracted binding to the program and score the resulting state against the source task parameters.
The checkers inspect application records or saved file contents to check whether the program has completed the task.

The verifier stops each check at the first failed task and records subsequent tasks as untested.
The repair loop uses the failed execution to analyze the error, complete the task with the agent, and refine the program.
Each refinement is one round of this repair loop.
We allow at most three refinements for each source task, at most nine in total, and charge each refinement to the task that first failed the check.
After refinement, we check the three tasks again from the beginning and stop repair once all three pass.
Four of five verified desktop and web programs require repair, with Qwen Writer the exception.
When an isolated browser worker cannot keep the failed execution's interface state, the agent restarts the source task and we record the restart and its full charge.
Verification uses the source tasks, while the thirty deployment tasks provide a separate evaluation.

We execute each program run in a worker that the parent process can terminate after the platform's deadline: $1200$ seconds on Android and desktop, and $900$ seconds on the web.
This external deadline remains effective when generated code catches an exception inside its own retry loop.
We record program-run timeouts separately from task-check failures, provider interruptions, and runs stopped by the operator.
Each program run starts from its task's restored initial state.

Builder and refinement calls request temperature zero and use provider-default output and reasoning limits across task families.
The Writer attempt and some Qwen repeated attempts use the official Qwen API because of interruptions on OpenRouter.
The adapter maps the recorded identifier to \texttt{qwen3.8-flash} and preserves the requested generation settings.
The official API documents a temperature floor of $0.6$ in its default thinking mode, so the requested value does not establish an identical effective sampling temperature across routes.\footnote{\url{https://platform.qianwenai.com/docs/developer-guides/getting-started/latest-model}}
The records identify the route used for each call.
For the Web DeepSeek attempt, we retried through Fireworks after three default-route replies contained no program text.
All three recovery responses reached an observed limit of $131{,}072$ completion tokens without producing program text.
We count this as failure to generate a program within that budget.
This is the limit at which the responses stopped, not an output limit we set in our requests.
Table~\ref{tab:compile-cost} includes these recorded recovery charges in the failed attempt's lower bound.

The compilation cost includes translation, the builder call, verification-time binding extraction, and every saved model call used to analyze failures, continue an agent run for repair, or refine the program.
It excludes the source agent runs because they are accounted for on the agent path.
Executing the program and checking its result make no model calls after binding extraction.
An unsuccessful attempt keeps every recorded translation, construction, extraction, and repair charge.
When call usage is missing, we report a saved lower bound rather than a complete attempt cost.

\looseness=-1
After a program passes verification, we deploy it on thirty fresh bindings.
Both verification and deployment use model-extracted bindings, but deployment uses fresh task prompts and a new extraction on each task.
The deployment failure fraction therefore includes failures in the complete extraction-and-execution path.

\subsection{Serving profiles}
\label{app:serving-profiles}
The exploration runs are the agent runs on the three source tasks that produce the source traces, with retries after failed runs.
Table~\ref{tab:share} contains the exploration and deployment profiles of the initial programs used in the experiments, and Table~\ref{tab:paired-replays} contains the matched agent replays summarized in Section~\ref{sec:paired}.
The profile means come from the exploration runs on the three source tasks and differ from the replay means on the thirty deployment bindings.
\begin{table}[htbp]
\centering
\caption{Serving costs and estimated savings for the initial programs.
Both cost columns use thousands of the tokens defined in Section~\ref{sec:formulation}.
$c$ is the baseline-subtracted mean agent cost.
Thirty deployment runs provide the mean extraction cost $c^{\mathrm{prog}}$ and failure fraction $q$.
The saving share charges one mean-cost agent fallback for a failed execution.
Blue, amber, and green identify GLM, DeepSeek, and Qwen, respectively.
Appendix~\ref{app:measurement-uncertainty} reports confidence intervals for the deployment estimates.
A dash means an unavailable program-path value, and the Qwen web request was refused by the provider.}
\label{tab:share}
\tablefont
\setlength{\tabcolsep}{2.5pt}
\renewcommand{\arraystretch}{0.95}
\begin{tabular}{@{}lrrrr@{}}
\toprule
\multicolumn{1}{c}{Family} & \multicolumn{1}{c}{\shortstack{$c$\\(K tokens)}} & \multicolumn{1}{c}{\shortstack{$c^{\mathrm{prog}}$\\(K tokens)}} & \multicolumn{1}{c}{\shortstack{$q$\\(fraction)}} & \multicolumn{1}{c}{\shortstack{Saving share\\(fraction)}} \\
\midrule
Contacts & \modelcell{2.7em}{98.7}{--}{39.2} & \modelcell{2.7em}{0.425}{--}{0.552} & \modelcell{2.1em}{0.000}{--}{0.033} & \modelcell{2.1em}{0.996}{--}{0.953} \\
Markor delete & \modelcell{2.7em}{33.7}{10.8}{21.5} & \modelcell{2.7em}{0.332}{0.331}{0.388} & \modelcell{2.1em}{0.000}{0.033}{0.000} & \modelcell{2.1em}{0.990}{0.936}{0.982} \\
Calendar & \modelcell{2.7em}{314}{--}{--} & \modelcell{2.7em}{0.640}{--}{--} & \modelcell{2.1em}{0.000}{--}{--} & \modelcell{2.1em}{0.998}{--}{--} \\
OsmAnd marker & \modelcell{2.7em}{--}{--}{--} & \modelcell{2.7em}{--}{--}{--} & \modelcell{2.1em}{--}{--}{--} & \modelcell{2.1em}{--}{--}{--} \\
\addlinespace[2pt]
Calc table & \modelcell{2.7em}{--}{647}{--} & \modelcell{2.7em}{--}{0.677}{--} & \modelcell{2.1em}{--}{0.000}{--} & \modelcell{2.1em}{--}{0.999}{--} \\
Writer memo & \modelcell{2.7em}{70.2}{50.7}{37.0} & \modelcell{2.7em}{0.596}{0.692}{0.724} & \modelcell{2.1em}{0.000}{0.000}{0.000} & \modelcell{2.1em}{0.992}{0.986}{0.980} \\
\addlinespace[2pt]
Reddit comment & \modelcell{2.7em}{124}{--}{--} & \modelcell{2.7em}{0.716}{--}{--} & \modelcell{2.1em}{0.000}{--}{--} & \modelcell{2.1em}{0.994}{--}{--} \\
\bottomrule
\end{tabular}
\end{table}

\begin{table}[htbp]
\vspace{-10pt}%
\centering
\caption{Agent replay costs and success counts on thirty matched deployment bindings per populated entry.
$c$ is the mean baseline-subtracted cost in K tokens, with K=$10^3$.
Saving share uses the mean-cost fallback estimate $qc$.
Blue, amber, and green identify GLM, DeepSeek, and Qwen, respectively.
Appendix~\ref{app:measurement-uncertainty} reports the confidence intervals.}
\label{tab:paired-replays}
\tablefont
\setlength{\tabcolsep}{1.5pt}
\renewcommand{\arraystretch}{0.95}
\begin{tabular}{@{}lrrrr@{}}
\toprule
\multicolumn{1}{c}{Family} & \multicolumn{1}{c}{\shortstack{Mean $c$\\(K tokens)}} & \multicolumn{1}{c}{\shortstack{Agent successes\\(of 30)}} & \multicolumn{1}{c}{\shortstack{Program successes\\(of 30)}} & \multicolumn{1}{c}{\shortstack{Saving share\\(fraction)}} \\
\midrule
Contacts & \modelcell{3.4em}{67.1}{--}{42.2} & \modelcell{2.9em}{30/30}{--}{30/30} & \modelcell{2.9em}{30/30}{--}{29/30} & \modelcell{2.9em}{0.994}{--}{0.954} \\
Markor delete & \modelcell{3.4em}{35.4}{22.3}{19.3} & \modelcell{2.9em}{30/30}{23/30}{30/30} & \modelcell{2.9em}{30/30}{29/30}{30/30} & \modelcell{2.9em}{0.991}{0.952}{0.980} \\
Calendar & \modelcell{3.4em}{535.3}{--}{--} & \modelcell{2.9em}{14/30}{--}{--} & \modelcell{2.9em}{30/30}{--}{--} & \modelcell{2.9em}{0.999}{--}{--} \\
OsmAnd marker & \modelcell{3.4em}{--}{--}{--} & \modelcell{2.9em}{--}{--}{--} & \modelcell{2.9em}{--}{--}{--} & \modelcell{2.9em}{--}{--}{--} \\
\bottomrule
\end{tabular}
\vspace{-14pt}%
\end{table}

\subsection{Repeated compilation}\label{app:measurement-checks}

We repeat compilation to check whether the outcome of the initial attempt persists.
We include two additional builder draws for each of the four Android families on GLM and DeepSeek, and for Contacts, Markor, and Calendar on Qwen.
These draws reuse the same three source tasks and saved translations as the initial attempt.
Each additional program goes through the same verification and repair as the initial program.
This gives twenty-two additional attempts and three observations for each family--model combination tested.
Twelve additional attempts pass verification, six fail the program checks, and four exhaust the builder's three-response retry budget without producing a valid program.
The generation failures include empty or syntactically invalid outputs and are distinct from failed program executions.

\begin{table}[htbp]
\centering
\caption{\looseness=-1 Repeated compilation outcomes.
Each entry reports the initial attempt and two repeated attempts as P (passes verification), F (fails verification), or G (generation failure).
Verification requires all three source tasks to pass, and tasks left untested after the first failed task are not counted as failures.
Repeated attempts are not deployed.
The given and extracted columns report the checks of the initial program in Appendix~\ref{app:given-extracted}, which are not part of verification.
Blue, amber, and green identify GLM, DeepSeek, and Qwen, respectively.
The three compilation attempts share traces and translation, so no independent-trial confidence interval is assigned.}
\label{tab:verification-repeat}
\tablefont
\setlength{\tabcolsep}{2.5pt}
\renewcommand{\arraystretch}{0.95}
\begin{tabular}{@{}lrrr@{}}
\toprule
\multicolumn{1}{c}{Family} & \multicolumn{1}{c}{Verification outcomes} & \multicolumn{1}{c}{Given (of 5)} & \multicolumn{1}{c}{Extracted (of 5)} \\
\midrule
Contacts & \modelcell{3.2em}{P/P/P}{F/F/P}{P/G/P} & \modelcell{2.2em}{0/5}{--}{5/5} & \modelcell{2.2em}{0/5}{--}{5/5} \\
Markor delete & \modelcell{3.2em}{P/P/P}{P/P/P}{P/P/P} & \modelcell{2.2em}{5/5}{5/5}{5/5} & \modelcell{2.2em}{5/5}{5/5}{5/5} \\
Calendar & \modelcell{3.2em}{P/P/P}{F/G/F}{F/G/G} & \modelcell{2.2em}{3/5}{--}{--} & \modelcell{2.2em}{3/5}{--}{--} \\
OsmAnd marker & \modelcell{3.2em}{F/F/F}{F/F/F}{--} & \modelcell{2.2em}{--}{--}{--} & \modelcell{2.2em}{--}{--}{--} \\
Calc table & \modelcell{3.2em}{--}{--}{--} & \modelcell{2.2em}{--}{--}{--} & \modelcell{2.2em}{--}{--}{--} \\
Writer memo & \modelcell{3.2em}{--}{--}{--} & \modelcell{2.2em}{--}{--}{--} & \modelcell{2.2em}{--}{--}{--} \\
Reddit comment & \modelcell{3.2em}{--}{--}{--} & \modelcell{2.2em}{--}{--}{--} & \modelcell{2.2em}{--}{--}{--} \\
\bottomrule
\end{tabular}
\end{table}

Markor passes all three attempts for each model, as do GLM Contacts and Calendar.
GLM and DeepSeek OsmAnd, DeepSeek Calendar, and Qwen Calendar produce no verified program in their three attempts.
DeepSeek Contacts fails initially and on its first repeat, then passes on its second repeat.
Qwen Contacts passes initially and on its second repeat, but its first repeat ends in a generation failure.
All verification runs use separate test environments with a matched Android emulator profile, app setup, and task resets.
Each environment has its own device state, and each cell has only three observations.
They do not establish independent retry outcomes or a stable within-family verification probability.
We report these initial and repeated attempts as counts, without a confidence interval for a fixed verification probability.

The call records also let us compare builder and repair charges between failed and successful attempts of one family.
Repeated builds reuse the initial translations, so this comparison excludes translation costs from both sides.
For DeepSeek Contacts, the failed initial attempt and first repeat cost $514.9$k and $888.6$k tokens before translation, while the successful second repeat costs $807.7$k.
The failed attempts therefore cost both less and more than the successful attempt under the same verification and accounting boundary.
They do not identify both conditional mean costs for the remaining families, whose missing-cost assumptions are tested in Section~\ref{app:cost-sensitivity}.

\subsection{Given and extracted parameters}
\label{app:given-extracted}
We check the six verified Android programs on five additional bindings, once with the benchmark's given parameter values and once with model-extracted values, to separate extraction errors from program errors.
Four programs pass all five tasks through both input paths.
GLM Calendar passes three of five, and GLM Contacts passes none, with either input path.
All extracted values pass the type check.
For every task that fails in this check, the extracted values match the given values and both input paths fail.

GLM Contacts completed all thirty deployment tasks but fails here on a rebuilt emulator instance.
The software keyboard stays visible after the program enters the name fields and hides the phone input field from the accessibility tree.
The program then selects the phone-type label field and enters the number there without raising an exception.
Dismissing the keyboard after each text entry lets the program complete both bindings tested in a follow-up check.

\subsection{Program failures under GUI changes}
\label{sec:fragility}
\looseness=-1
This test measures how often GUI changes make a verified program fail and whether the failures are detected.
We test verified programs in a clean environment and under nine GUI changes adapted from two robustness benchmarks \citep{lee2025bmoca,chen2026dgara}.
The test runs programs without model calls.
Among bindings that pass in the clean environment, it measures the fraction that fail under each GUI change.
The five configuration changes use font scales of $1.3$ and $0.85$, display density $0.85$ times the default, French locale, and dark theme.
The four interruptions are a notification, low battery, a permission dialog, and an update prompt implemented as a notification-based stand-in.
Device fingerprints check that each change is reverted before the next condition.
We match each changed-condition outcome to the same binding in the clean condition and exclude skipped changes.
Programs with no clean passes are excluded from these conditional rates.
Averaging across the remaining programs gives appearance-change failure fractions of $0$, $0$, and $0.100$ on GLM, DeepSeek, and Qwen, respectively.
Interruptions give corresponding fractions of $0.842$, $1.000$, and $0.833$.
GLM Contacts has no clean passes because of the keyboard state described in Section~\ref{app:given-extracted}, but its failures remain in the pooled detectability calculation.
Across all test failures, including failures in the clean condition, detectable execution errors account for about $58$, $87$, and $71$ percent on the three models.
The remaining failures reach an incorrect task state without an execution error, which limits error-triggered fallback.
These undetected failures show that the detected-failure assumption of the cost model in Section~\ref{sec:formulation} does not always hold.
The GUI change failure rates describe the selected changes and tested families.
They do not estimate how frequently a deployment encounters those changes, so the online evaluation treats the drift hazard as a given scenario parameter.

\subsection{Measurement confidence intervals}
\label{app:measurement-uncertainty}
Cost intervals resample the recorded binding indices, keeping the association between agent cost, extraction cost, and program outcome.
They condition on the recorded program and binding protocol.
Success and failure intervals use an independent Bernoulli working model.
Zero failures in thirty trials gives a two-sided $95\%$ interval of $[0,0.116]$, so a narrow saving-share bootstrap cannot rule out unseen failure modes.
The exploration runs, which retry failed source tasks, the single cost record of each initial compilation, and the three attempts of each repeated-compilation cell do not support independent-trial cost intervals.
The matched replays use a different emulator profile from the program runs and therefore do not isolate model randomness.
\begin{table}[H]
\centering
\caption{Confidence intervals for deployment estimates of the initial programs.
$c^{\mathrm{prog}}$ uses tokens and gives the mean followed by positive and negative offsets of its 95 percent binding-bootstrap interval.
Failures shows an observed count followed by the lower and upper endpoints of a 95 percent Clopper--Pearson interval under an independent Bernoulli working model.
Blue, amber, and green identify GLM, DeepSeek, and Qwen, respectively.
These intervals condition on the recorded program and task-binding protocol.
The exploration runs, which retry failed source tasks, and single compilation cost records do not support a cost interval based on independent trials.}
\label{tab:deployment-uncertainty}
\tablefont
\setlength{\tabcolsep}{4pt}
\renewcommand{\arraystretch}{0.95}
\begin{tabular}{@{}lrr@{}}
\toprule
\multicolumn{1}{c}{Family} & \multicolumn{1}{c}{\shortstack{$c^{\mathrm{prog}}$\\(tokens)}} & \multicolumn{1}{c}{\shortstack{Failures (of 30)\\95\% CI (fraction)}} \\
\midrule
Contacts & \modelcell{3.4em}{\estci{424.5}{+14.2}{-23.0}}{\statpoint{--}}{\estci{552.1}{+39.6}{-35.9}} & \modelcell{3.4em}{\countci{0/30}{0.000}{0.116}}{\statpoint{--}}{\countci{1/30}{0.001}{0.172}} \\
Markor delete & \modelcell{3.4em}{\estci{331.6}{+5.1}{-4.8}}{\estci{330.6}{+13.1}{-13.2}}{\estci{388.1}{+25.4}{-23.6}} & \modelcell{3.4em}{\countci{0/30}{0.000}{0.116}}{\countci{1/30}{0.001}{0.172}}{\countci{0/30}{0.000}{0.116}} \\
Calendar & \modelcell{3.4em}{\estci{640.0}{+42.8}{-47.9}}{\statpoint{--}}{\statpoint{--}} & \modelcell{3.4em}{\countci{0/30}{0.000}{0.116}}{\statpoint{--}}{\statpoint{--}} \\
OsmAnd marker & \modelcell{3.4em}{\statpoint{--}}{\statpoint{--}}{\statpoint{--}} & \modelcell{3.4em}{\statpoint{--}}{\statpoint{--}}{\statpoint{--}} \\
Calc table & \modelcell{3.4em}{\statpoint{--}}{\estci{676.6}{+12.9}{-13.1}}{\statpoint{--}} & \modelcell{3.4em}{\statpoint{--}}{\countci{0/30}{0.000}{0.116}}{\statpoint{--}} \\
Writer memo & \modelcell{3.4em}{\estci{596.3}{+52.9}{-46.3}}{\estci{692.0}{+42.5}{-42.8}}{\estci{724.1}{+55.7}{-50.6}} & \modelcell{3.4em}{\countci{0/30}{0.000}{0.116}}{\countci{0/30}{0.000}{0.116}}{\countci{0/30}{0.000}{0.116}} \\
Reddit comment & \modelcell{3.4em}{\estci{716.4}{+105.6}{-87.9}}{\statpoint{--}}{\statpoint{--}} & \modelcell{3.4em}{\countci{0/30}{0.000}{0.116}}{\statpoint{--}}{\statpoint{--}} \\
\bottomrule
\end{tabular}
\end{table}

\begin{table}[H]
\centering
\caption{Agent-cost estimates and confidence intervals for the matched replays in Table~\ref{tab:paired-replays}.
Each populated entry is the mean cost across thirty recorded bindings.
The offsets below the mean give its 95 percent binding-bootstrap interval in K tokens.
CV divides the sample standard deviation by the mean.
Blue, amber, and green identify GLM, DeepSeek, and Qwen, respectively.}
\label{tab:paired-cost-uncertainty}
\tablefont
\setlength{\tabcolsep}{4pt}
\renewcommand{\arraystretch}{0.95}
\begin{tabular}{@{}lrr@{}}
\toprule
\multicolumn{1}{c}{Family} & \multicolumn{1}{c}{\shortstack{Mean $c$\\(K tokens)}} & \multicolumn{1}{c}{\shortstack{Agent CV\\(ratio)}} \\
\midrule
Contacts & \modelcell{3.8em}{\estci{67.1}{+7.8}{-7.8}}{\statpoint{--}}{\estci{42.2}{+1.5}{-1.4}} & \modelcell{2.8em}{\statpoint{0.326}}{\statpoint{--}}{\statpoint{0.099}} \\
Markor delete & \modelcell{3.8em}{\estci{35.4}{+4.8}{-4.5}}{\estci{22.3}{+9.4}{-7.1}}{\estci{19.3}{+1.0}{-.8}} & \modelcell{2.8em}{\statpoint{0.373}}{\statpoint{1.042}}{\statpoint{0.128}} \\
Calendar & \modelcell{3.8em}{\estci{535.3}{+94.4}{-89.3}}{\statpoint{--}}{\statpoint{--}} & \modelcell{2.8em}{\statpoint{0.485}}{\statpoint{--}}{\statpoint{--}} \\
OsmAnd marker & \modelcell{3.8em}{\statpoint{--}}{\statpoint{--}}{\statpoint{--}} & \modelcell{2.8em}{\statpoint{--}}{\statpoint{--}}{\statpoint{--}} \\
\bottomrule
\end{tabular}
\end{table}

\begin{table}[H]
\centering
\caption{Confidence intervals for success rates and saving shares on the matched replay bindings.
Successes gives observed counts and 95 percent Clopper--Pearson endpoints under an independent Bernoulli working model.
Saving share gives the mean with positive and negative 95 percent bootstrap offsets, resampling matched agent costs, program charges, and failure indicators together.
Blue, amber, and green identify GLM, DeepSeek, and Qwen, respectively.
The saving bootstrap cannot reveal a failure mode absent from the sample.
An offset marked $<.001$ is positive but smaller than $0.001$.
For zero failures in thirty uses, the binomial failure interval is $[0,0.116]$ even when the empirical saving interval is narrow.}
\label{tab:paired-uncertainty}
\tablefont
\setlength{\tabcolsep}{1.2pt}
\renewcommand{\arraystretch}{0.95}
\begin{tabular}{@{}lrrr@{}}
\toprule
\multicolumn{1}{c}{Family} & \multicolumn{1}{c}{\shortstack{Agent successes (of 30)\\95\% CI (fraction)}} & \multicolumn{1}{c}{\shortstack{Program successes (of 30)\\95\% CI (fraction)}} & \multicolumn{1}{c}{\shortstack{Saving share\\(fraction)}} \\
\midrule
Contacts & \modelcell{3.4em}{\countci{30/30}{0.884}{1.000}}{\statpoint{--}}{\countci{30/30}{0.884}{1.000}} & \modelcell{3.4em}{\countci{30/30}{0.884}{1.000}}{\statpoint{--}}{\countci{29/30}{0.828}{0.999}} & \modelcell{3.4em}{\estci{0.994}{+{<}.001}{-{<}.001}}{\statpoint{--}}{\estci{0.954}{+.034}{-.067}} \\
Markor delete & \modelcell{3.4em}{\countci{30/30}{0.884}{1.000}}{\countci{23/30}{0.577}{0.901}}{\countci{30/30}{0.884}{1.000}} & \modelcell{3.4em}{\countci{30/30}{0.884}{1.000}}{\countci{29/30}{0.828}{0.999}}{\countci{30/30}{0.884}{1.000}} & \modelcell{3.4em}{\estci{0.991}{+.001}{-.001}}{\estci{0.952}{+.037}{-.071}}{\estci{0.980}{+.002}{-.002}} \\
Calendar & \modelcell{3.4em}{\countci{14/30}{0.283}{0.657}}{\statpoint{--}}{\statpoint{--}} & \modelcell{3.4em}{\countci{30/30}{0.884}{1.000}}{\statpoint{--}}{\statpoint{--}} & \modelcell{3.4em}{\estci{0.999}{+{<}.001}{-{<}.001}}{\statpoint{--}}{\statpoint{--}} \\
OsmAnd marker & \modelcell{3.4em}{\statpoint{--}}{\statpoint{--}}{\statpoint{--}} & \modelcell{3.4em}{\statpoint{--}}{\statpoint{--}}{\statpoint{--}} & \modelcell{3.4em}{\statpoint{--}}{\statpoint{--}}{\statpoint{--}} \\
\bottomrule
\end{tabular}
\end{table}

\begin{table}[H]
\centering
\caption{Checks of the verified programs on five additional bindings, with 95 percent Clopper--Pearson intervals.
Each cell shows the observed count and then the lower and upper endpoints under an independent Bernoulli working model.
Blue, amber, and green identify GLM, DeepSeek, and Qwen, respectively.
Five checks provide limited evidence about future binding reliability.
Verification and repair use their bindings adaptively and are reported as exact counts without test-set intervals.}
\label{tab:binding-check-uncertainty}
\tablefont
\setlength{\tabcolsep}{4pt}
\renewcommand{\arraystretch}{0.95}
\begin{tabular}{@{}lrr@{}}
\toprule
\multicolumn{1}{c}{Family} & \multicolumn{1}{c}{\shortstack{Given (of 5)\\95\% CI (fraction)}} & \multicolumn{1}{c}{\shortstack{Extracted (of 5)\\95\% CI (fraction)}} \\
\midrule
Contacts & \modelcell{3.4em}{\countci{0/5}{0.000}{0.522}}{\statpoint{--}}{\countci{5/5}{0.478}{1.000}} & \modelcell{3.4em}{\countci{0/5}{0.000}{0.522}}{\statpoint{--}}{\countci{5/5}{0.478}{1.000}} \\
Markor delete & \modelcell{3.4em}{\countci{5/5}{0.478}{1.000}}{\countci{5/5}{0.478}{1.000}}{\countci{5/5}{0.478}{1.000}} & \modelcell{3.4em}{\countci{5/5}{0.478}{1.000}}{\countci{5/5}{0.478}{1.000}}{\countci{5/5}{0.478}{1.000}} \\
Calendar & \modelcell{3.4em}{\countci{3/5}{0.147}{0.947}}{\statpoint{--}}{\statpoint{--}} & \modelcell{3.4em}{\countci{3/5}{0.147}{0.947}}{\statpoint{--}}{\statpoint{--}} \\
OsmAnd marker & \modelcell{3.4em}{\statpoint{--}}{\statpoint{--}}{\statpoint{--}} & \modelcell{3.4em}{\statpoint{--}}{\statpoint{--}}{\statpoint{--}} \\
\bottomrule
\end{tabular}
\end{table}

\section{Online algorithm details}\label{app:protocol-details}

\looseness=-1
Algorithm~\ref{alg:policy} states the execution order, including required source traces, program list control, fallback, and both compilation outcomes.
\begin{algorithm}[htbp]
\caption{PACE with extra-cost margin $\epsilon$.
Family identities, cost profiles, routing charges, and the probability of GUI drift are given.}
\label{alg:policy}
\begin{algorithmic}[1]
\State Initialize cumulative charges $K\gets0$ and agent reference $A\gets0$
\For{each arrival $t$}
\State Load the arriving family's costs and state
\State $A\gets A+\tau_0+c$
\State Remove idle list entries and add the family's saved program if it can still be reused before the budget check
\State Compute $U_t^{\mathrm{serv}}$ from Equation~\ref{eq:service-reserve}
\If{$K+U_t^{\mathrm{serv}}>(1+\epsilon)A$}
\State Clear the program list and select agent service
\EndIf
\State Pay routing and serve the task; on a detected program failure, pay one agent fallback
\State Add the routing and serving charges to $K$; log completed agent traces
\State Stop reusing the program after detected GUI drift and update the task arrival records
\If{at least $k$ agent runs are complete, there is no program to reuse, and Equation~\ref{eq:trigger} holds}
\If{$K+\max(C,C^{\mathrm{fail}})\leq(1+\epsilon)A$}
\State Compile a program and verify it on completed source tasks; increment $j$
\If{verification passes}
\State $K\gets K+C$; increment $g$; keep and list the program
\Else
\State $K\gets K+C^{\mathrm{fail}}$
\EndIf
\EndIf
\EndIf
\EndFor
\end{algorithmic}
\end{algorithm}

\subsection{Compilation spend}
\label{app:reference}
For independent attempts with fixed verification probability $p>0$, the number of failed attempts before the first successful verification has mean $(1-p)/p$.
The expected compilation spend is therefore
\begin{equation}
B(p)=C+\left(\frac1p-1\right)C^{\mathrm{fail}}.
\label{eq:compile-spend}
\end{equation}
Using the estimate from past outcomes $\hat p$ gives PACE's estimated compilation cost $\hat B$.
This calculation excludes agent service between attempts and assumes a fixed retry distribution, which the repeated GUI measurements do not confirm.
The cost guarantee instead counts each actual model charge and does not use this independence assumption.

\subsection{Program lifetime}
\label{sec:program-lifetime}
A replacement calculation explains the lifetime cap $1/h$.
Suppose a working program can be produced at fixed cost $B$.
On each execution, GUI drift independently makes it stop working with probability $h>0$.
A replacement cycle has expected length $1/h$, including the execution on which GUI drift occurs.
It pays $B$, extraction per execution, fallback with probability $q$ conditional on no GUI drift, and one fallback when GUI drift occurs.
Dividing expected cycle cost by expected length gives
\begin{equation}
c^{\mathrm{prog}}+(1-h)qc+h(B+c).
\label{eq:replacement}
\end{equation}
This is below agent cost $c$ exactly when $s_h/h>B$.
At $\hat H=1/h$ with no routing charge, the proposal has the same form, using estimated spend $\hat B$.
Smaller reuse estimates further limit the savings available to repay compilation.
This reference calculation compares two specified policies and does not establish optimal replacement under unknown parameters.

\subsection{Cost of reserving for failure}
\label{app:reservation-example}
Reserving the failed-attempt cost can delay a compilation until observed arrivals add enough budget.
E.g., suppose a proposal passes the savings test in Equation~\ref{eq:trigger}, but success costs $60$ and failure $120$ tokens.
With reference and charged costs both $300$, $\epsilon=0.25$ allows only $75$ additional tokens, so PACE waits to compile.
Two more agent-served arrivals costing $100$ each raise both totals to $500$.
The resulting $125$ allowance covers either outcome.
Only observed arrivals increase this allowance.

When the failed-attempt cost exceeds the allowance for the whole stream, the reservation prevents compilation entirely.
A single-family example has $300$ arrivals, $k=3$, agent cost $1$, successful compilation cost $2$, and possible failed cost $200$.
Set $c^{\mathrm{prog}}=q=h=\tau_0=m=0$.
The compiler succeeds, but the policy does not know that in advance.
With $\epsilon=0.25$, PACE's extra budget never reaches the $200$ reservation, so it uses the agent and pays $300$.
Eager retry pays $3$ for source runs and $2$ for compilation.
The example shows how reserving the failure branch can prevent a profitable compilation while staying within the cost bound relative to agent execution.

\section{Online evaluation details}\label{app:simulation-details}

\subsection{Cost profiles}
\label{app:cost-profiles}

We construct cost profiles for the online evaluation from measurements of seven families for GLM and DeepSeek and six families with compilation records for Qwen.
The refused Qwen web request provides no cost profile.
For a cell with a successful compilation, we use its agent cost, extraction cost, and successful compilation cost.
We use its deployment failure fraction as $q$ conditional on no GUI drift.
Each test setting defines $h$ separately.
A cell without a successful compilation uses medians from cells with successful compilations under the same model for the unobserved successful-program quantities.
The agent success probability is the fraction of completed exploration attempts that succeeded, including unsuccessful attempts in the denominator.

\looseness=-1
A complete failed compilation gives the failed-attempt cost for that cell.
For an unobserved failed outcome, we use the same model's median complete failed-attempt cost.
For Web DeepSeek, whose recorded charge is incomplete, we set this estimate to at least its recorded cost lower bound.
We exclude incomplete charges from the donor median.
These values define fixed charges in the experiments.
The estimates of missing costs are assumptions about unobserved outcomes, not additional measurements.

\subsection{Synthetic task streams}
We evaluate two grids of $300$ conditions each.
Each includes all combinations of three model profiles, four verification settings, five arrival patterns, and five cost or environment settings.
The first grid has $1200$ arrivals and $30$ families that appear repeatedly, except in its sparse pattern.
Its verification probabilities are $0.9$, $0.5$, $0$, or an independently assigned per-family mixture over $\{0,0.1,0.5,0.9,1\}$.
Its patterns are uniform family sampling, Zipf sampling, bursts, a change in the frequent family, and infrequent repeats.
The uniform pattern concerns family labels, not wall-clock Poisson arrivals.
The base setting uses $h=0.02$, an idle limit of $T=100$, and router cost $368+91\ell_t$.
We measure routing cost with program lists of up to $100$ entries and use the resulting intercept and rounded per-entry cost.
We use these constants as shared scenario parameters across model profiles.
The other settings multiply only failed-compilation cost by five, multiply both compilation costs by five, set $h=0.1$, or raise the per-entry routing charge to $1820$.

\looseness=-1
The second grid uses its own random seeds and parameters.
It has $1800$ arrivals and $40$ families that appear repeatedly, except in its sparse pattern, with verification probabilities $0.95$, $0.35$, $0$, or a mixture over $\{0,0.05,0.3,0.7,1\}$.
Its altered-cost settings multiply failed cost by three or both compilation costs by seven, and its drift and routing settings use $h=0.07$ and $m=910$.
We report every condition in both grids.
The probabilities are assigned independently of the profiles' initial verification outcomes, so an estimate for a missing value does not identify a family's verification probability.

\subsection{Recorded arrival orders}\label{app:arrival-orders}
We also use Sepsis and BPI 2019 process logs \citep{mannhardt2016sepsis,vandongen2019bpi} and two Wikipedia edit streams.
In each process log, an arrival is a completed case, ordered by its earliest event, and a family is its full ordered activity sequence.
Sepsis contains $1050$ arrivals and $846$ families, while BPI 2019 contains $251{,}734$ arrivals and $11{,}973$ families.
These activity-sequence labels are assigned after each case completes and do not test online family identification.

\looseness=-1
For Wikipedia, an arrival is an edit, and families group normalized edit-summary templates with bot and tool labels.
The tool stream contains $65{,}564$ arrivals and $2529$ families.
We use the first $100{,}000$ timestamp-ordered human edits, which contain $68{,}773$ families.
Each human edit with an empty or section-only summary gets its own family label.
The source window spans 3 September 2026 at 18:48:26 UTC to 5 September 2026 at 18:48:26 UTC, as recorded in \path{datasets/wiki-stream/structure_report.md}.
The exact summary normalization is in \path{datasets/wiki-stream/build_families.py}, and \path{code/t2sim/streams.py} defines sorting and stream selection.

\looseness=-1
The logs give task frequencies and arrival order rather than live GUI tasks.
For each run, measured cost profiles are shuffled onto log-family identifiers, and verification probabilities are assigned independently from the first grid's mixture.
Three models and four logs give twelve combinations of models and logs.
The runs use the same base drift and routing settings as the first grid.

\paragraph{Confidence intervals.}
Each grid condition has eight paired runs, and each combination of a model and a log has ten.
Compilation outcomes are indexed by family and attempt, while serving outcomes are indexed by family, arrival, and failure channel.
Rules in a run share these outcome sequences and the same cost mapping.
A cell's cost ratio is the policy's mean token cost divided by the PACE mean in that cell.
We compute the overall result by averaging the cell ratios with equal weights, without combining absolute costs across models.
The grid and log intervals use $2000$ paired bootstrap resamples, with run indices shared across policies and conditions in each study.
They describe variation in the saved mappings and outcomes, conditional on the tested scenarios and given cost profiles.
They do not account for measurement error in the given costs or the effects of choosing different estimates of missing costs.
The largest tested condition mean is a descriptive maximum and has no confidence interval.

We used both grids while developing PACE, so the reported comparisons are exploratory evidence rather than independent confirmation.
All grid and log conditions use detected failures and one agent fallback.
The experimental environment reports whether a failure is due to GUI drift.
The policy uses this report to stop program reuse.
We do not test an online drift detector.
When a program fails, the model uses the same agent outcome as ReAct for that arrival.
This rule ensures that task success in the model cannot be lower than ReAct's.
The cost bound itself makes no such quality claim.

\subsection{Comparison rules}\label{app:baselines}
\label{sec:baselines}

\looseness=-1
We compare PACE with ReAct, two decision strategies adapted from published systems, and alternative timing rules with and without the same cumulative budget.
All methods use the same measured compiler and serving profiles, so the comparison isolates their compilation and reuse decisions.
The named adaptations do not reproduce the original systems' models, code generators, or benchmark scores.

\paragraph{ReAct.}
The ReAct baseline uses agent execution on every arrival and never compiles or reuses programs \citep{yao2023react}.

\paragraph{AutoRPA (one build).}
AutoRPA constructs a program from source tasks, limits its repair attempts, and uses ReAct when construction fails \citep{chen2026autorpa}.
Our adaptation attempts one build after $k=3$ completed agent runs, with the measured compilation charge already including bounded repair.
A failed build leaves the family on the agent path.
If the build succeeds, the program serves later tasks with agent fallback, and detected GUI drift returns the family to agent service without another build.
This transfers the bounded construction-and-reuse policy to the shared compiler and arrival model.

\paragraph{ToolPro (cost adaptation).}
ToolPro compares construction latency with the latency saved within one task \citep{liu2026toolpro}.
We adapt this single-task criterion to token cost: a family with enough traces and no program to reuse requests compilation only when $s_h>\hat B$.
The estimate $\hat B$ uses the same observed verification outcomes as PACE.
This rule gives no credit for savings on multiple future tasks.
We keep the shared parameterized programs and serving order, so the comparison tests the cost criterion rather than ToolPro's original tool execution system.

\paragraph{Cost-aware timing controls.}
\looseness=-1
The remaining controls share a per-family budget that allows compilation only when $F+C^{\mathrm{fail}}\leq\max(0,S)+\max(0,\hat Hs_h)$, where $F$ is failed-compilation spend and $S$ is accumulated program-use saving.
The allowance uses estimated future savings and does not enforce PACE's cumulative bound.
\emph{Eager + allowance} requests a build when the family has enough traces and no program to reuse.
\emph{After 10 arrivals} waits for ten arrivals of the family.
\emph{After 10 successes} waits for ten successful agent runs since the last compilation attempt or detected GUI drift.
\emph{Savings threshold} waits until the expected savings lost by using the agent instead of reusing a program reach $\hat B$.
\emph{Estimate + allowance} uses PACE's savings test with this allowance in place of its cost budget.

\paragraph{The same budget with other proposals.}
\looseness=-1
The budgeted rows in Tables~\ref{tab:policysim-full} and~\ref{tab:realstreams-full} hold PACE's $0.25$ margin, service reservations, and program list controls fixed.
\emph{Eager + both checks}, \emph{Arrival-10 + both}, and \emph{Estimate + both} add that budget to the corresponding estimate-based allowance rules.
\emph{Count + budget} changes only PACE's estimated reuse count to $\min(N,1/h)$.
\emph{Optimistic + budget} keeps PACE's estimate $\hat H$ but uses the one-sided $90$ percent Clopper--Pearson upper value in place of $\hat p$, using it as an optimistic proposal estimate.
\emph{One initial try} adds a limit to that optimistic rule: an unsuccessful first build ends compilation for the family, while an initially successful family can later attempt replacement.
The budget guarantee does not require these proposal estimates to be accurate.

All methods require $k=3$ completed agent runs before compilation and serve the current arrival before attempting a build.
They share free program lookup, idle expiry, adding programs to the list again if they can still be reused, and the same failure outcomes.
PACE and the other budgeted rules check service and compilation costs, and can clear the program list before a routing charge.
PACE's main setting is $\epsilon=0.25$ throughout both comparison tables.
Section~\ref{sec:ablation} separately removes each of its main components.

\subsection{Cost-profile sensitivity}\label{app:cost-sensitivity}

The base profiles contain eleven observed successful compilation costs, eight complete failed-attempt costs, and one failed-attempt cost lower bound.
We vary the cost of the unobserved outcome by setting target ratios $r=C^{\mathrm{fail}}/C$ to $0.5$, $1$, $2$, or $5$, while keeping the cost of each observed outcome unchanged.
For a profile with an observed successful compilation, we vary $C^{\mathrm{fail}}$.
For a profile with a complete failed attempt, we vary $C$.
For Web DeepSeek, we keep the estimate of the missing $C$ and set the failed cost to the larger of its recorded lower bound and $rC$.
Its actual ratio therefore exceeds the target value in all four settings.
The other profile values and the base probability assignments remain fixed.
These four settings give forty-eight combinations of models and logs.

We add two dependence scenarios using the base costs.
Verification probability decreases from $0.9$ to $0.1$ as the profile's rank in $C/c$ increases.
One scenario keeps the base profile assignments.
The other assigns profiles with high $C/c$ to frequent families without changing which profiles are used or how many times each is assigned.
Only the code that creates the workload uses counts from the full stream.
The policy receives neither those counts nor the hidden verification probabilities.
We fixed these seventy-two conditions before evaluation and saved them all.
They test alternatives to the base assumptions rather than estimate how these factors vary together in an actual deployment.

We separately test the incomplete Web DeepSeek charge at its recorded lower bound and at twice the modeled estimate, keeping all other settings fixed.
We use the default bursty condition and the BPI 2019 condition, giving four additional conditions that are excluded from the main comparison averages.

\begin{table}[htbp]
\centering
\caption{Sensitivity to missing costs and assigned cost dependencies.
Each row uses twelve combinations of models and logs and ten paired runs.
The ReAct, PACE, and Equal attempt costs columns show mean cost ratios to PACE for each model.
The last column averages the Equal attempt costs-to-PACE ratios across all models.
Nonzero-width 95 percent paired bootstrap intervals appear as upper and lower offsets below the mean; exact zero-width intervals are omitted.
Blue, amber, and green identify GLM, DeepSeek, and Qwen, respectively.
An offset marked $<.001$ is positive but smaller than $0.001$.
The listed $r=C^{\mathrm{fail}}/C$ is a target ratio.
The incomplete DeepSeek Web charge sets a recorded lower bound on its failed-attempt cost, which can raise the actual ratio.
Both policies use the same $\epsilon=0.25$ budget.}
\label{tab:missing-cost-sensitivity}
\tablefont
\setlength{\tabcolsep}{2pt}
\renewcommand{\arraystretch}{0.95}
\begin{tabular}{@{}lrrrr@{}}
\toprule
\multicolumn{1}{c}{\shortstack{Scenario\\$r$: cost ratio}} & \multicolumn{1}{c}{\shortstack{ReAct\\($\times$ PACE)}} & \multicolumn{1}{c}{\shortstack{PACE\\($\times$ PACE)}} & \multicolumn{1}{c}{\shortstack{Equal attempt costs\\($\times$ PACE)}} & \multicolumn{1}{c}{\shortstack{Equal / PACE\\(all models)}} \\
\midrule
$r=0.5$ & \modelcell{3.4em}{\estci{1.280}{+.113}{-.072}}{\estci{1.226}{+.068}{-.060}}{\estci{1.066}{+.024}{-.017}} & \modelcell{3.4em}{\statpoint{1.000}}{\statpoint{1.000}}{\statpoint{1.000}} & \modelcell{3.4em}{\estci{0.998}{+.007}{-.005}}{\estci{1.012}{+.011}{-.009}}{\estci{1.027}{+.022}{-.015}} & \estci{1.012}{+.007}{-.006} \\
$r=1$ & \modelcell{3.4em}{\estci{1.322}{+.150}{-.090}}{\estci{1.241}{+.091}{-.076}}{\estci{1.131}{+.063}{-.040}} & \modelcell{3.4em}{\statpoint{1.000}}{\statpoint{1.000}}{\statpoint{1.000}} & \modelcell{3.4em}{\statpoint{1.000}}{\estci{1.005}{+.001}{-.001}}{\statpoint{1.000}} & \estci{1.002}{+{<}.001}{-{<}.001} \\
$r=2$ & \modelcell{3.4em}{\estci{1.334}{+.179}{-.098}}{\estci{1.241}{+.107}{-.085}}{\estci{1.177}{+.086}{-.056}} & \modelcell{3.4em}{\statpoint{1.000}}{\statpoint{1.000}}{\statpoint{1.000}} & \modelcell{3.4em}{\estci{1.030}{+.011}{-.009}}{\estci{1.032}{+.011}{-.011}}{\estci{0.999}{+.020}{-.022}} & \estci{1.020}{+.010}{-.009} \\
$r=5$ & \modelcell{3.4em}{\estci{1.261}{+.151}{-.085}}{\estci{1.210}{+.105}{-.079}}{\estci{1.191}{+.085}{-.060}} & \modelcell{3.4em}{\statpoint{1.000}}{\statpoint{1.000}}{\statpoint{1.000}} & \modelcell{3.4em}{\estci{1.106}{+.031}{-.027}}{\estci{1.174}{+.045}{-.038}}{\estci{1.143}{+.055}{-.050}} & \estci{1.141}{+.034}{-.030} \\
Cost-dependent $p$ & \modelcell{3.4em}{\estci{1.601}{+.142}{-.134}}{\estci{2.046}{+.153}{-.160}}{\estci{1.783}{+.163}{-.135}} & \modelcell{3.4em}{\statpoint{1.000}}{\statpoint{1.000}}{\statpoint{1.000}} & \modelcell{3.4em}{\estci{1.141}{+.056}{-.052}}{\estci{1.083}{+.022}{-.018}}{\estci{1.144}{+.093}{-.074}} & \estci{1.123}{+.038}{-.032} \\
Cost-dependent $p$ + costly frequent & \modelcell{3.4em}{\estci{0.958}{+.001}{-.001}}{\statpoint{1.000}}{\statpoint{1.000}} & \modelcell{3.4em}{\statpoint{1.000}}{\statpoint{1.000}}{\statpoint{1.000}} & \modelcell{3.4em}{\estci{0.989}{+{<}.001}{-{<}.001}}{\estci{1.100}{+.002}{-.002}}{\estci{1.078}{+.002}{-.002}} & \estci{1.056}{+.001}{-.002} \\
\bottomrule
\end{tabular}
\end{table}

Across the four missing-cost settings, PACE's mean cost ranges from $0.833$ to $0.866$ of ReAct.
The Equal attempt costs variant costs $1.020$ and $1.141$ times PACE on average at $r=2$ and $r=5$.
PACE costs more than this variant in six of twelve conditions at $r=0.5$ and in two at $r=2$, but less in all twelve at $r=5$.
At $r=1$, the eight GLM and Qwen conditions tie exactly.
For DeepSeek, the partial web spending record sets a minimum failed-attempt cost, so differences remain in the other four conditions.
The value of using separate costs therefore depends on the assumed difference between successful and failed compilation costs and how that difference changes compilation timing.

\looseness=-1
Under cost-dependent verification alone, PACE has a mean cost of $0.671$ times ReAct cost.
Adding the assignment of costly profiles to frequent families raises that ratio to $1.015$, so PACE costs about $1.5\%$ more than ReAct on average.
Equal attempt costs has a mean cost ratio to PACE of $1.056$ in the same joint scenario.
Both variants satisfy the implemented budget check, with a relative numerical tolerance of $10^{-8}$, in all $1440$ runs.
The largest recorded prefix ratio is $1.25000001$.
These results separate the cost constraint from the amount saved, which depends on the workload and given cost model.

In the separate Web DeepSeek missing-charge check, PACE-to-ReAct cost ratios range from $0.686$ to $0.785$ on the default bursty condition and from $0.545$ to $0.618$ on BPI 2019.
The corresponding ratios under the modeled failed cost are $0.734$ and $0.624$.
All four additional conditions satisfy the implemented budget.

\subsection{Alternative compilation proposals}
\label{sec:matched-budget}

Applying PACE's budget to the cost-aware timing controls preserves their low costs.
On the recorded logs, Eager + both checks, Arrival-10 + both, and Estimate + both have mean costs of $0.8349$, $0.8323$, and $0.8312$ times ReAct cost, compared with $0.8266$ for PACE.
Each paired $95$ percent interval for control cost divided by PACE cost lies above one, although each mean cost difference is below $1\%$.

Count + budget has mean cost ratios to ReAct of $0.9556$ and $0.9536$ on the two grids, compared with $0.9563$ and $0.9539$ for PACE.
On the recorded logs, its mean cost is $0.8265$ times ReAct cost.
Its paired $95\%$ interval for the cost ratio to PACE is $[0.9997,0.9999]$.
The count rule is slightly cheaper than PACE's estimate $\hat H$ in these conditions.
Optimistic + budget and One initial try have higher log means of $0.9203$ and $0.8839$.
Together, these results support cost-aware proposals within the shared cost budget, but do not identify the exact reuse-count estimate as a separate source of improvement.

\subsection{Complete cost comparisons}
\label{app:full-comparisons}
\looseness=-1
The following tables report the named adaptations and all other compilation rules, with conditional paired bootstrap intervals.
PACE's ratio is one in every run by normalization.
ToolPro matches ReAct in all conditions shown, so those rows have identical ratios and intervals relative to PACE.
We omit offsets only for zero-width intervals, not for small nonzero intervals.
Maxima are descriptive maxima over tested condition means and have no sampling interval.
\begin{table}[!htbp]
\centering
\caption{Cost across the two parameter grids, relative to PACE.
Each grid contains one hundred conditions per model and eight paired runs per condition.
Mean is the mean of the cost ratios for individual conditions, and Max is the largest tested condition mean.
Nonzero-width 95 percent paired bootstrap intervals appear as upper and lower offsets below the mean; exact zero-width intervals are omitted.
Intervals use 2000 shared bootstrap resamples of runs and condition on the tested scenarios and given profiles.
An offset marked $<.001$ is positive but smaller than $0.001$.
Blue, amber, and green identify GLM, DeepSeek, and Qwen, respectively.
AutoRPA and ToolPro are decision-rule adaptations defined in Appendix~\ref{app:baselines}.
Tables~\ref{tab:policysim} and~\ref{tab:realstreams} show the rows for ReAct, AutoRPA, ToolPro, and PACE.
PACE is exactly one by normalization; other fixed ratios also have zero-width intervals.
The descriptive maxima have no sampling interval.}
\label{tab:policysim-full}
\tablefont
\setlength{\tabcolsep}{0.6pt}
\renewcommand{\arraystretch}{0.95}
\begin{tabular}{@{}lrrrr@{}}
\toprule
\multicolumn{1}{c}{Policy} & \multicolumn{1}{c}{\shortstack{Grid 1 mean\\($\times$ PACE) $\downarrow$}} & \multicolumn{1}{c}{\shortstack{Grid 1 max\\($\times$ PACE) $\downarrow$}} & \multicolumn{1}{c}{\shortstack{Grid 2 mean\\($\times$ PACE) $\downarrow$}} & \multicolumn{1}{c}{\shortstack{Grid 2 max\\($\times$ PACE) $\downarrow$}} \\
\midrule
ReAct & \modelcell{3.4em}{\estci{1.121}{+.023}{-.016}}{\estci{1.144}{+.014}{-.015}}{\estci{1.056}{+.022}{-.016}} & \modelcell{3.4em}{2.796}{3.495}{2.281} & \modelcell{3.4em}{\estci{1.147}{+.029}{-.022}}{\estci{1.174}{+.024}{-.021}}{\estci{1.070}{+.018}{-.014}} & \modelcell{3.4em}{3.692}{3.128}{2.742} \\
AutoRPA & \modelcell{3.4em}{\estci{2.107}{+.022}{-.006}}{\estci{1.942}{+.033}{-.021}}{\estci{2.715}{+.051}{-.044}} & \modelcell{3.4em}{11.997}{9.660}{19.314} & \modelcell{3.4em}{\estci{2.168}{+.031}{-.015}}{\estci{2.006}{+.051}{-.030}}{\estci{2.805}{+.045}{-.034}} & \modelcell{3.4em}{16.034}{12.927}{26.538} \\
ToolPro & \modelcell{3.4em}{\estci{1.121}{+.023}{-.016}}{\estci{1.144}{+.014}{-.015}}{\estci{1.056}{+.022}{-.016}} & \modelcell{3.4em}{2.796}{3.495}{2.281} & \modelcell{3.4em}{\estci{1.147}{+.029}{-.022}}{\estci{1.174}{+.024}{-.021}}{\estci{1.070}{+.018}{-.014}} & \modelcell{3.4em}{3.692}{3.128}{2.742} \\
Eager + allowance & \modelcell{3.4em}{\estci{1.036}{+.005}{-.004}}{\estci{1.070}{+.004}{-.003}}{\estci{1.007}{+.002}{-.003}} & \modelcell{3.4em}{1.449}{1.492}{1.314} & \modelcell{3.4em}{\estci{1.033}{+.004}{-.003}}{\estci{1.093}{+.005}{-.003}}{\estci{1.007}{+.003}{-.003}} & \modelcell{3.4em}{1.455}{1.525}{1.321} \\
After 10 arrivals & \modelcell{3.4em}{\estci{1.032}{+.005}{-.004}}{\estci{1.032}{+.004}{-.004}}{\estci{1.007}{+.002}{-.003}} & \modelcell{3.4em}{1.449}{1.492}{1.314} & \modelcell{3.4em}{\estci{1.029}{+.004}{-.003}}{\estci{1.055}{+.005}{-.003}}{\estci{1.007}{+.003}{-.003}} & \modelcell{3.4em}{1.455}{1.525}{1.321} \\
After 10 successes & \modelcell{3.4em}{\estci{1.068}{+.008}{-.006}}{\estci{1.083}{+.007}{-.007}}{\estci{1.028}{+.006}{-.004}} & \modelcell{3.4em}{1.716}{1.414}{1.323} & \modelcell{3.4em}{\estci{1.073}{+.007}{-.005}}{\estci{1.113}{+.008}{-.008}}{\estci{1.034}{+.008}{-.007}} & \modelcell{3.4em}{1.592}{1.505}{1.451} \\
Savings threshold & \modelcell{3.4em}{\estci{1.039}{+.005}{-.005}}{\estci{1.043}{+.006}{-.005}}{\estci{1.017}{+.003}{-.003}} & \modelcell{3.4em}{1.375}{1.409}{1.293} & \modelcell{3.4em}{\estci{1.050}{+.003}{-.003}}{\estci{1.061}{+.005}{-.006}}{\estci{1.021}{+.005}{-.004}} & \modelcell{3.4em}{1.383}{1.423}{1.309} \\
Estimate + allowance & \modelcell{3.4em}{\estci{1.008}{+{<}.001}{-{<}.001}}{\estci{1.014}{+.001}{-.001}}{\estci{1.007}{+.002}{-.002}} & \modelcell{3.4em}{1.367}{1.387}{1.286} & \modelcell{3.4em}{\estci{1.016}{+.002}{-.002}}{\estci{1.016}{+.002}{-.001}}{\estci{1.010}{+.002}{-.002}} & \modelcell{3.4em}{1.371}{1.410}{1.301} \\
\midrule
Eager + both checks & \modelcell{3.4em}{\estci{1.024}{+.003}{-.003}}{\estci{1.038}{+.006}{-.004}}{\estci{0.998}{+{<}.001}{-{<}.001}} & \modelcell{3.4em}{1.163}{1.240}{1.014} & \modelcell{3.4em}{\estci{1.015}{+.003}{-.002}}{\estci{1.049}{+.004}{-.003}}{\estci{0.998}{+.003}{-.003}} & \modelcell{3.4em}{1.109}{1.243}{1.066} \\
Arrival-10 + both & \modelcell{3.4em}{\estci{1.019}{+.003}{-.002}}{\estci{1.006}{+.004}{-.003}}{\estci{0.998}{+{<}.001}{-{<}.001}} & \modelcell{3.4em}{1.163}{1.107}{1.014} & \modelcell{3.4em}{\estci{1.011}{+.003}{-.002}}{\estci{1.020}{+.003}{-.002}}{\estci{0.998}{+.003}{-.003}} & \modelcell{3.4em}{1.092}{1.212}{1.066} \\
Estimate + both & \modelcell{3.4em}{\estci{1.000}{+{<}.001}{-{<}.001}}{\estci{1.000}{+{<}.001}{-{<}.001}}{\estci{1.000}{+{<}.001}{-{<}.001}} & \modelcell{3.4em}{1.025}{1.050}{1.008} & \modelcell{3.4em}{\estci{1.001}{+.002}{-.001}}{\estci{1.000}{+{<}.001}{-{<}.001}}{\estci{1.000}{+.001}{-{<}.001}} & \modelcell{3.4em}{1.116}{1.028}{1.064} \\
Count + budget & \modelcell{3.4em}{\estci{0.999}{+{<}.001}{-{<}.001}}{\estci{0.999}{+{<}.001}{-{<}.001}}{\estci{0.999}{+{<}.001}{-{<}.001}} & \modelcell{3.4em}{1.003}{1.003}{1.008} & \modelcell{3.4em}{\estci{1.000}{+{<}.001}{-{<}.001}}{\estci{0.998}{+{<}.001}{-{<}.001}}{\estci{1.000}{+{<}.001}{-{<}.001}} & \modelcell{3.4em}{1.006}{1.004}{1.039} \\
Optimistic + budget & \modelcell{3.4em}{\estci{1.115}{+.008}{-.006}}{\estci{1.114}{+.007}{-.005}}{\estci{1.128}{+.012}{-.008}} & \modelcell{3.4em}{1.359}{1.348}{1.442} & \modelcell{3.4em}{\estci{1.081}{+.005}{-.004}}{\estci{1.110}{+.009}{-.006}}{\estci{1.079}{+.009}{-.007}} & \modelcell{3.4em}{1.384}{1.421}{1.412} \\
One initial try & \modelcell{3.4em}{\estci{1.115}{+.014}{-.014}}{\estci{1.123}{+.024}{-.021}}{\estci{1.126}{+.013}{-.011}} & \modelcell{3.4em}{1.736}{2.003}{1.554} & \modelcell{3.4em}{\estci{1.065}{+.009}{-.007}}{\estci{1.101}{+.019}{-.012}}{\estci{1.068}{+.008}{-.008}} & \modelcell{3.4em}{1.503}{1.967}{1.294} \\
\textbf{PACE} & \modelcell{3.4em}{\statpoint{1.000}}{\statpoint{1.000}}{\statpoint{1.000}} & \modelcell{3.4em}{1.000}{1.000}{1.000} & \modelcell{3.4em}{\statpoint{1.000}}{\statpoint{1.000}}{\statpoint{1.000}} & \modelcell{3.4em}{1.000}{1.000}{1.000} \\
\bottomrule
\end{tabular}
\end{table}

\begin{table}[!htbp]
\centering
\caption{Cost on four recorded arrival sequences, relative to PACE.
Each value divides the policy's mean cost by PACE's mean cost in the same condition, using ten paired runs.
Nonzero-width 95 percent paired bootstrap intervals appear as upper and lower offsets below the mean; exact zero-width intervals are omitted.
Intervals use 2000 shared bootstrap resamples of runs and condition on the tested scenarios and given profiles.
An offset marked $<.001$ is positive but smaller than $0.001$.
Blue, amber, and green identify GLM, DeepSeek, and Qwen, respectively.
AutoRPA and ToolPro are decision-rule adaptations defined in Appendix~\ref{app:baselines}.
Tables~\ref{tab:policysim} and~\ref{tab:realstreams} show the rows for ReAct, AutoRPA, ToolPro, and PACE.
PACE is exactly one by normalization; other fixed ratios also have zero-width intervals.}
\label{tab:realstreams-full}
\tablefont
\setlength{\tabcolsep}{0.6pt}
\renewcommand{\arraystretch}{0.95}
\begin{tabular}{@{}lrrrr@{}}
\toprule
\multicolumn{1}{c}{Policy} & \multicolumn{1}{c}{\shortstack{Sepsis\\($\times$ PACE) $\downarrow$}} & \multicolumn{1}{c}{\shortstack{BPI 2019\\($\times$ PACE) $\downarrow$}} & \multicolumn{1}{c}{\shortstack{Wiki tools\\($\times$ PACE) $\downarrow$}} & \multicolumn{1}{c}{\shortstack{Wiki humans\\($\times$ PACE) $\downarrow$}} \\
\midrule
ReAct & \modelcell{3.4em}{\estci{0.971}{+.019}{-.022}}{\estci{0.998}{+.046}{-.048}}{\estci{0.989}{+.014}{-.013}} & \modelcell{3.4em}{\estci{1.531}{+.272}{-.183}}{\estci{1.603}{+.299}{-.222}}{\estci{1.393}{+.179}{-.144}} & \modelcell{3.4em}{\estci{1.730}{+.669}{-.387}}{\estci{1.480}{+.428}{-.247}}{\estci{1.288}{+.303}{-.166}} & \modelcell{3.4em}{\estci{1.056}{+.017}{-.015}}{\estci{1.060}{+.016}{-.016}}{\estci{1.040}{+.013}{-.012}} \\
AutoRPA & \modelcell{3.4em}{\estci{1.145}{+.033}{-.029}}{\estci{1.123}{+.036}{-.037}}{\estci{1.246}{+.045}{-.041}} & \modelcell{3.4em}{\estci{1.582}{+.278}{-.190}}{\estci{1.641}{+.303}{-.225}}{\estci{1.489}{+.181}{-.146}} & \modelcell{3.4em}{\estci{1.754}{+.676}{-.392}}{\estci{1.498}{+.429}{-.251}}{\estci{1.331}{+.317}{-.172}} & \modelcell{3.4em}{\estci{1.164}{+.017}{-.015}}{\estci{1.144}{+.014}{-.015}}{\estci{1.232}{+.016}{-.014}} \\
ToolPro & \modelcell{3.4em}{\estci{0.971}{+.019}{-.022}}{\estci{0.998}{+.046}{-.048}}{\estci{0.989}{+.014}{-.013}} & \modelcell{3.4em}{\estci{1.531}{+.272}{-.183}}{\estci{1.603}{+.299}{-.222}}{\estci{1.393}{+.179}{-.144}} & \modelcell{3.4em}{\estci{1.730}{+.669}{-.387}}{\estci{1.480}{+.428}{-.247}}{\estci{1.288}{+.303}{-.166}} & \modelcell{3.4em}{\estci{1.056}{+.017}{-.015}}{\estci{1.060}{+.016}{-.016}}{\estci{1.040}{+.013}{-.012}} \\
Eager + allowance & \modelcell{3.4em}{\estci{1.007}{+.007}{-.006}}{\estci{1.027}{+.016}{-.019}}{\estci{0.997}{+.001}{-.001}} & \modelcell{3.4em}{\estci{1.002}{+.003}{-.002}}{\estci{0.995}{+.012}{-.016}}{\estci{1.057}{+.072}{-.056}} & \modelcell{3.4em}{\estci{1.004}{+.009}{-.006}}{\estci{0.997}{+.030}{-.038}}{\estci{1.003}{+.005}{-.006}} & \modelcell{3.4em}{\estci{1.005}{+.002}{-.002}}{\estci{1.021}{+.001}{-.001}}{\estci{1.001}{+{<}.001}{-.001}} \\
After 10 arrivals & \modelcell{3.4em}{\estci{1.010}{+.007}{-.006}}{\estci{1.010}{+.007}{-.008}}{\estci{0.997}{+.001}{-.001}} & \modelcell{3.4em}{\estci{1.003}{+.004}{-.003}}{\estci{0.988}{+.012}{-.017}}{\estci{1.057}{+.072}{-.056}} & \modelcell{3.4em}{\estci{1.005}{+.009}{-.006}}{\estci{0.994}{+.030}{-.038}}{\estci{1.003}{+.005}{-.006}} & \modelcell{3.4em}{\estci{1.006}{+.002}{-.002}}{\estci{1.008}{+.010}{-.005}}{\estci{1.001}{+{<}.001}{-.001}} \\
After 10 successes & \modelcell{3.4em}{\estci{1.007}{+.009}{-.009}}{\estci{1.011}{+.016}{-.016}}{\estci{1.003}{+.007}{-.006}} & \modelcell{3.4em}{\estci{1.123}{+.053}{-.040}}{\estci{1.131}{+.061}{-.052}}{\estci{1.147}{+.107}{-.068}} & \modelcell{3.4em}{\estci{1.241}{+.235}{-.139}}{\estci{1.128}{+.098}{-.070}}{\estci{1.068}{+.057}{-.032}} & \modelcell{3.4em}{\estci{1.024}{+.004}{-.004}}{\estci{1.025}{+.009}{-.007}}{\estci{1.015}{+.005}{-.004}} \\
Savings threshold & \modelcell{3.4em}{\estci{1.000}{+{<}.001}{-{<}.001}}{\estci{1.004}{+.005}{-.004}}{\estci{0.999}{+{<}.001}{-{<}.001}} & \modelcell{3.4em}{\estci{1.066}{+.029}{-.021}}{\estci{1.057}{+.034}{-.028}}{\estci{1.072}{+.033}{-.031}} & \modelcell{3.4em}{\estci{1.070}{+.053}{-.031}}{\estci{1.056}{+.053}{-.041}}{\estci{1.025}{+.016}{-.009}} & \modelcell{3.4em}{\estci{1.010}{+.002}{-.002}}{\estci{1.016}{+.009}{-.006}}{\estci{1.008}{+.002}{-.002}} \\
Estimate + allowance & \modelcell{3.4em}{\statpoint{1.000}}{\statpoint{1.000}}{\statpoint{1.000}} & \modelcell{3.4em}{\estci{0.999}{+{<}.001}{-{<}.001}}{\estci{1.002}{+.006}{-.003}}{\estci{1.057}{+.071}{-.055}} & \modelcell{3.4em}{\estci{1.000}{+.007}{-.005}}{\estci{1.013}{+.022}{-.015}}{\estci{1.004}{+.005}{-.005}} & \modelcell{3.4em}{\estci{0.999}{+{<}.001}{-{<}.001}}{\estci{1.005}{+.010}{-.005}}{\estci{1.000}{+{<}.001}{-{<}.001}} \\
\midrule
Eager + both checks & \modelcell{3.4em}{\estci{1.007}{+.007}{-.006}}{\estci{1.027}{+.016}{-.019}}{\estci{0.997}{+.001}{-.001}} & \modelcell{3.4em}{\estci{1.003}{+.003}{-.002}}{\estci{0.995}{+.012}{-.016}}{\estci{1.057}{+.072}{-.056}} & \modelcell{3.4em}{\estci{1.004}{+.009}{-.006}}{\estci{0.998}{+.031}{-.038}}{\estci{1.000}{+.004}{-.003}} & \modelcell{3.4em}{\estci{1.005}{+.002}{-.002}}{\estci{1.021}{+.001}{-.001}}{\estci{1.001}{+{<}.001}{-.001}} \\
Arrival-10 + both & \modelcell{3.4em}{\estci{1.010}{+.007}{-.006}}{\estci{1.010}{+.007}{-.008}}{\estci{0.997}{+.001}{-.001}} & \modelcell{3.4em}{\estci{1.003}{+.004}{-.003}}{\estci{0.988}{+.012}{-.017}}{\estci{1.057}{+.072}{-.056}} & \modelcell{3.4em}{\estci{1.005}{+.010}{-.006}}{\estci{0.997}{+.031}{-.039}}{\estci{1.000}{+.004}{-.003}} & \modelcell{3.4em}{\estci{1.006}{+.002}{-.002}}{\estci{1.008}{+.010}{-.005}}{\estci{1.001}{+{<}.001}{-.001}} \\
Estimate + both & \modelcell{3.4em}{\statpoint{1.000}}{\statpoint{1.000}}{\statpoint{1.000}} & \modelcell{3.4em}{\estci{0.999}{+{<}.001}{-{<}.001}}{\estci{1.002}{+.006}{-.003}}{\estci{1.057}{+.071}{-.055}} & \modelcell{3.4em}{\estci{1.001}{+.007}{-.005}}{\estci{1.012}{+.023}{-.014}}{\estci{1.001}{+.004}{-.003}} & \modelcell{3.4em}{\estci{0.999}{+{<}.001}{-{<}.001}}{\estci{1.005}{+.010}{-.005}}{\estci{1.000}{+{<}.001}{-{<}.001}} \\
Count + budget & \modelcell{3.4em}{\statpoint{1.000}}{\estci{1.000}{+{<}.001}{-{<}.001}}{\estci{0.999}{+{<}.001}{-{<}.001}} & \modelcell{3.4em}{\estci{1.000}{+{<}.001}{-{<}.001}}{\estci{1.000}{+{<}.001}{-{<}.001}}{\estci{1.000}{+{<}.001}{-{<}.001}} & \modelcell{3.4em}{\estci{0.999}{+{<}.001}{-{<}.001}}{\estci{1.000}{+.001}{-.001}}{\estci{0.999}{+{<}.001}{-{<}.001}} & \modelcell{3.4em}{\estci{1.000}{+{<}.001}{-{<}.001}}{\estci{1.000}{+{<}.001}{-{<}.001}}{\estci{1.000}{+{<}.001}{-{<}.001}} \\
Optimistic + budget & \modelcell{3.4em}{\estci{1.159}{+.032}{-.035}}{\estci{1.106}{+.032}{-.028}}{\estci{1.206}{+.017}{-.016}} & \modelcell{3.4em}{\estci{1.093}{+.025}{-.022}}{\estci{1.058}{+.014}{-.014}}{\estci{1.042}{+.049}{-.054}} & \modelcell{3.4em}{\estci{1.067}{+.046}{-.029}}{\estci{1.008}{+.041}{-.061}}{\estci{0.964}{+.090}{-.099}} & \modelcell{3.4em}{\estci{1.157}{+.007}{-.006}}{\estci{1.123}{+.007}{-.008}}{\estci{1.248}{+.007}{-.006}} \\
One initial try & \modelcell{3.4em}{\estci{1.075}{+.032}{-.028}}{\estci{1.041}{+.027}{-.030}}{\estci{1.136}{+.036}{-.037}} & \modelcell{3.4em}{\estci{1.069}{+.060}{-.041}}{\estci{1.067}{+.047}{-.047}}{\estci{1.066}{+.094}{-.082}} & \modelcell{3.4em}{\estci{1.076}{+.100}{-.051}}{\estci{1.096}{+.088}{-.077}}{\estci{0.989}{+.030}{-.046}} & \modelcell{3.4em}{\estci{1.057}{+.005}{-.005}}{\estci{1.041}{+.007}{-.006}}{\estci{1.108}{+.007}{-.006}} \\
\textbf{PACE} & \modelcell{3.4em}{\statpoint{1.000}}{\statpoint{1.000}}{\statpoint{1.000}} & \modelcell{3.4em}{\statpoint{1.000}}{\statpoint{1.000}}{\statpoint{1.000}} & \modelcell{3.4em}{\statpoint{1.000}}{\statpoint{1.000}}{\statpoint{1.000}} & \modelcell{3.4em}{\statpoint{1.000}}{\statpoint{1.000}}{\statpoint{1.000}} \\
\bottomrule
\end{tabular}
\end{table}

\subsection{Component costs}
\label{app:ablation-absolute}
\looseness=-1
Table~\ref{tab:ablation-absolute} gives cost ratios relative to PACE for the ablations, and Table~\ref{tab:ablation-effects-full} gives their paired percentage changes with intervals.
A change is the mean of $100(K_{\mathrm{variant}}/K_{\mathrm{PACE}}-1)$ over fixed combinations of models and conditions, using each condition's mean costs.
Effects can interact and are not additive shares of a total saving.
PACE's change is exactly zero because we compare PACE with itself.
On the recorded logs, equal attempt costs raises mean cost for each model.
Removing the budget has little effect on mean cost on the recorded logs, while its worst grid cost increases.
By contrast, the savings-test and cost-estimate ablations keep the same budget, which keeps their maximum cost ratios to ReAct near the same limit.
Equal attempt costs is cheaper than PACE in $62$ and $60$ conditions of the two grids, so its average loss is not uniform across conditions.
\begin{table}[htbp]
\centering
\caption{Cost ratios for the component ablations, relative to PACE.
The first three columns show the mean of the cost ratios to PACE computed separately for each condition, with positive and negative offsets of the 95 percent paired bootstrap interval.
Grid max is the largest tested condition mean across both grids and all three models.
Blue, amber, and green identify GLM, DeepSeek, and Qwen, respectively.
An offset marked $<.001$ is positive but smaller than $0.001$.
Equal attempt costs assumes $C^{\mathrm{fail}}=C$ when estimating compilation cost, while keeping actual charges and reservations unchanged.
No cost budget removes reservations and budget-triggered program list clearing.}
\label{tab:ablation-absolute}
\tablefont
\setlength{\tabcolsep}{1.0pt}
\renewcommand{\arraystretch}{0.95}
\begin{tabular}{@{}lrrrr@{}}
\toprule
\multicolumn{1}{c}{Variant} & \multicolumn{1}{c}{\shortstack{Grid 1 mean\\($\times$ PACE)}} & \multicolumn{1}{c}{\shortstack{Grid 2 mean\\($\times$ PACE)}} & \multicolumn{1}{c}{\shortstack{Logs mean\\($\times$ PACE)}} & \multicolumn{1}{c}{\shortstack{Grid max\\($\times$ PACE)}} \\
\midrule
\textbf{PACE} & \modelcell{3.4em}{\statpoint{1.000}}{\statpoint{1.000}}{\statpoint{1.000}} & \modelcell{3.4em}{\statpoint{1.000}}{\statpoint{1.000}}{\statpoint{1.000}} & \modelcell{3.4em}{\statpoint{1.000}}{\statpoint{1.000}}{\statpoint{1.000}} & \statpoint{1.000} \\
No savings test & \modelcell{3.4em}{\estci{1.174}{+.011}{-.011}}{\estci{1.186}{+.014}{-.010}}{\estci{1.181}{+.017}{-.011}} & \modelcell{3.4em}{\estci{1.140}{+.012}{-.007}}{\estci{1.167}{+.012}{-.009}}{\estci{1.121}{+.014}{-.009}} & \modelcell{3.4em}{\estci{1.581}{+.149}{-.087}}{\estci{1.548}{+.086}{-.064}}{\estci{1.451}{+.092}{-.068}} & \statpoint{1.957} \\
No cost budget & \modelcell{3.4em}{\estci{1.008}{+{<}.001}{-{<}.001}}{\estci{1.014}{+.001}{-.001}}{\estci{1.007}{+.001}{-.002}} & \modelcell{3.4em}{\estci{1.015}{+.001}{-{<}.001}}{\estci{1.016}{+{<}.001}{-{<}.001}}{\estci{1.010}{+.002}{-.001}} & \modelcell{3.4em}{\estci{1.000}{+{<}.001}{-{<}.001}}{\estci{1.000}{+{<}.001}{-{<}.001}}{\estci{1.000}{+{<}.001}{-{<}.001}} & \statpoint{1.410} \\
Equal attempt costs & \modelcell{3.4em}{\estci{1.100}{+.009}{-.006}}{\estci{1.093}{+.008}{-.006}}{\estci{1.119}{+.010}{-.006}} & \modelcell{3.4em}{\estci{1.082}{+.004}{-.002}}{\estci{1.095}{+.010}{-.005}}{\estci{1.077}{+.010}{-.006}} & \modelcell{3.4em}{\estci{1.227}{+.041}{-.032}}{\estci{1.175}{+.067}{-.059}}{\estci{1.297}{+.056}{-.055}} & \statpoint{1.568} \\
Eager retry & \modelcell{3.4em}{\estci{10.317}{+.481}{-.262}}{\estci{8.414}{+.745}{-.301}}{\estci{18.329}{+.659}{-.458}} & \modelcell{3.4em}{\estci{10.919}{+.555}{-.532}}{\estci{9.352}{+.614}{-.341}}{\estci{19.286}{+.694}{-.467}} & \modelcell{3.4em}{\estci{3.303}{+.490}{-.437}}{\estci{3.326}{+.566}{-.524}}{\estci{4.922}{+.587}{-.576}} & \statpoint{195.631} \\
\bottomrule
\end{tabular}
\end{table}

\begin{table}[htbp]
\centering
\caption{Cost change after removing or simplifying PACE components.
$\Delta$ is the mean percentage cost change from full PACE across conditions. Positive values mean higher cost.
Blue, amber, and green identify GLM, DeepSeek, and Qwen, respectively.
Offsets show the conditional 95 percent paired interval in percentage points; exact zero-width intervals are omitted.
Grid max is the largest tested condition-mean ratio, without a sampling interval.}
\label{tab:ablation-effects-full}
\tablefont
\setlength{\tabcolsep}{0.6pt}
\renewcommand{\arraystretch}{0.95}
\begin{tabular}{@{}lrrrr@{}}
\toprule
\multicolumn{1}{c}{Variant} & \multicolumn{1}{c}{\shortstack{Grid 1 $\Delta$ cost\\(\%)}} & \multicolumn{1}{c}{\shortstack{Grid 2 $\Delta$ cost\\(\%)}} & \multicolumn{1}{c}{\shortstack{Logs $\Delta$ cost\\(\%)}} & \multicolumn{1}{c}{\shortstack{Grid max\\($\times$ PACE)}} \\
\midrule
\textbf{PACE} & \modelcell{3.8em}{\statpoint{0.0}}{\statpoint{0.0}}{\statpoint{0.0}} & \modelcell{3.8em}{\statpoint{0.0}}{\statpoint{0.0}}{\statpoint{0.0}} & \modelcell{3.8em}{\statpoint{0.0}}{\statpoint{0.0}}{\statpoint{0.0}} & \modelcell{3.4em}{1.000}{1.000}{1.000} \\
No savings test & \modelcell{3.8em}{\estci{17.4}{+1.1}{-1.1}}{\estci{18.6}{+1.4}{-1.0}}{\estci{18.1}{+1.7}{-1.1}} & \modelcell{3.8em}{\estci{14.0}{+1.2}{-.7}}{\estci{16.7}{+1.2}{-.9}}{\estci{12.1}{+1.4}{-.9}} & \modelcell{3.8em}{\estci{58.1}{+14.9}{-8.7}}{\estci{54.8}{+8.6}{-6.4}}{\estci{45.1}{+9.2}{-6.8}} & \modelcell{3.4em}{1.957}{1.810}{1.901} \\
No cost budget & \modelcell{3.8em}{\estci{0.8}{+{<}.1}{-{<}.1}}{\estci{1.4}{+.1}{-.1}}{\estci{0.7}{+.1}{-.2}} & \modelcell{3.8em}{\estci{1.5}{+.1}{-{<}.1}}{\estci{1.6}{+{<}.1}{-{<}.1}}{\estci{1.0}{+.2}{-.1}} & \modelcell{3.8em}{\estci{0.016}{+.080}{-.044}}{\estci{-0.013}{+.010}{-.015}}{\estci{0.002}{+.016}{-.013}} & \modelcell{3.4em}{1.371}{1.410}{1.301} \\
Equal attempt costs & \modelcell{3.8em}{\estci{10.0}{+.9}{-.6}}{\estci{9.3}{+.8}{-.6}}{\estci{11.9}{+1.0}{-.6}} & \modelcell{3.8em}{\estci{8.2}{+.4}{-.2}}{\estci{9.5}{+1.0}{-.5}}{\estci{7.7}{+1.0}{-.6}} & \modelcell{3.8em}{\estci{22.7}{+4.1}{-3.2}}{\estci{17.5}{+6.7}{-5.9}}{\estci{29.7}{+5.6}{-5.5}} & \modelcell{3.4em}{1.568}{1.567}{1.466} \\
Eager retry & \modelcell{3.8em}{\estci{931.7}{+48.1}{-26.2}}{\estci{741.4}{+74.5}{-30.1}}{\estci{1732.9}{+65.9}{-45.8}} & \modelcell{3.8em}{\estci{991.9}{+55.5}{-53.2}}{\estci{835.2}{+61.4}{-34.1}}{\estci{1828.6}{+69.4}{-46.7}} & \modelcell{3.8em}{\estci{230.3}{+49.0}{-43.7}}{\estci{232.6}{+56.6}{-52.4}}{\estci{392.2}{+58.7}{-57.6}} & \modelcell{3.4em}{92.863}{73.996}{195.631} \\
\bottomrule
\end{tabular}
\end{table}

\section{Live GUI execution}
\label{app:live}

We test agent fallback and call-cost reservations on the OpenApps Calendar web application \citep{ullrich2025openapps}.
The fixed program was compiled by GLM 5.3 Flash from one agent trace on the same calendar application before this study.
Six runs without model calls confirmed task completion on the unchanged interface and failure under a renamed submit button or a blocking dialog.
This study uses a scheduled compilation attempt and does not evaluate PACE's online compilation choice.

\subsection{Paired task execution}
Four seeds specify event title, date, description, location, URL, and invitees.
For each seed, we test a clean interface, a renamed submit button, and a blocking dialog, giving twelve agent/program pairs with matching initial states.
Both paths use GLM 5.3 Flash and the same natural-language goal.
The agent reads screenshots and an accessibility tree in headless Chromium.
The program receives a model-extracted binding and uses Playwright locators.
An execution error triggers agent fallback from the current GUI state, without resetting it.
A task has a total limit of twenty model calls.
We stop the task if a model call fails or its cost cannot be confirmed.
The evaluator checks the created event and whether unrelated application state remains unchanged.

\begin{table}[htbp]
\centering
\caption{Observed outcomes and serving costs for twelve paired Calendar tasks.
Mean cost bounds are in milli-USD and include extraction and any fallback.
Their endpoints combine known charges with the saved reservations for unconfirmed costs; they are not confidence intervals.
Interrupted runs remain unsuccessful outcomes, and the agent has one additional step-limit failure.
Serving costs exclude the compilation of the fixed program and the separate compilation attempt.}
\label{tab:live-main}
\tablefont
\setlength{\tabcolsep}{4pt}
\begin{tabular}{@{}lrrr@{}}
\toprule
\multicolumn{1}{c}{Serving path} & \multicolumn{1}{c}{Success (of 12)} & \multicolumn{1}{c}{Interrupted (count)} & \multicolumn{1}{c}{\shortstack{Mean cost bounds\\(milli-USD)}} \\
\midrule
Agent & 7/12 & 4 & [8.142, 21.590] \\
Program + fallback & 11/12 & 1 & [1.202, 4.564] \\
\bottomrule
\end{tabular}
\end{table}

The agent succeeds in seven of twelve tasks, with four call interruptions and one step-limit failure.
All four clean program runs succeed without fallback.
All eight program runs with changed interfaces use fallback, and seven succeed after keeping all six target fields at handoff.
The remaining run has no correct target fields at handoff and stops on a call interruption.
In only five pairs with changed interfaces do both policies reach the changed part before the runs end.
Both succeed in all five.
Results from these five pairs do not replace the comparison across all twelve pairs.
Tables~\ref{tab:live-pairs} and~\ref{tab:live-tight} keep all outcomes, including interrupted and denied tasks.

\begin{table}[htbp]
\centering
\caption{All paired tasks specified in advance, including provider interruptions and unsuccessful execution.
Costs use milli-USD.
Bracketed values give lower and upper accounting bounds, and costs shown without brackets are confirmed.
P+F means program execution followed by agent fallback on an execution error.
Fields counts the six target values already correct when fallback begins.
All pairs start from matching application and interface states.}
\label{tab:live-pairs}
\tablefont
\setlength{\tabcolsep}{2pt}
\begin{tabular}{@{}lrrrrrr@{}}
\toprule
\multicolumn{1}{c}{Seed} & \multicolumn{1}{c}{Condition} & \multicolumn{1}{c}{Agent} & \multicolumn{1}{c}{P+F} & \multicolumn{1}{c}{\shortstack{Agent cost\\(milli-USD)}} & \multicolumn{1}{c}{\shortstack{P+F cost\\(milli-USD)}} & \multicolumn{1}{c}{Fields (of 6)} \\
\midrule
0 & Clean & Interrupted & Success & [5.147, 45.493] & 0.048 & -- \\
0 & Renamed submit & Interrupted & Success & [2.050, 42.396] & 0.763 & 6/6 \\
0 & Blocking dialog & Step limit & Success & 17.156 & 1.255 & 6/6 \\
1 & Clean & Success & Success & 9.687 & 0.048 & -- \\
1 & Renamed submit & Success & Success & 6.831 & 0.763 & 6/6 \\
1 & Blocking dialog & Success & Success & 11.091 & 1.129 & 6/6 \\
2 & Clean & Interrupted & Success & [14.067, 54.413] & 0.048 & -- \\
2 & Renamed submit & Success & Success & 8.658 & 0.891 & 6/6 \\
2 & Blocking dialog & Interrupted & Interrupted & [0.000, 40.346] & [7.354, 47.699] & 0/6 \\
3 & Clean & Success & Success & 6.802 & 0.049 & -- \\
3 & Renamed submit & Success & Success & 7.876 & 0.893 & 6/6 \\
3 & Blocking dialog & Success & Success & 8.334 & 1.184 & 6/6 \\
\bottomrule
\end{tabular}
\end{table}

\begin{table}[htbp]
\centering
\caption{Four tight-budget tests in their arrival order set in advance.
Costs and prefix totals use milli-USD.
The reference increases by 20 milli-USD per scheduled arrival, including a denied arrival.
The prefix total includes charges and saved reservations for unsettled calls.
The first task is denied before any model call or GUI service, so this test does not satisfy the theorem's assumption that default service costs at most its reference charge.}
\label{tab:live-tight}
\tablefont
\setlength{\tabcolsep}{3pt}
\begin{tabular}{@{}lrrrrr@{}}
\toprule
\multicolumn{1}{c}{Arrival} & \multicolumn{1}{c}{Condition} & \multicolumn{1}{c}{Outcome} & \multicolumn{1}{c}{\shortstack{Cost\\(milli-USD)}} & \multicolumn{1}{c}{\shortstack{Prefix ceiling\\(milli-USD)}} & \multicolumn{1}{c}{\shortstack{Prefix total\\(milli-USD)}} \\
\midrule
1 & Renamed submit & Denied & 0.000 & 25.000 & 0.000 \\
2 & Blocking dialog & Success & 1.257 & 50.000 & 1.257 \\
3 & Renamed submit & Success & 0.890 & 75.000 & 2.147 \\
4 & Blocking dialog & Interrupted & [0.215, 40.561] & 100.000 & 42.708 \\
\bottomrule
\end{tabular}
\end{table}

\subsection{Reservations}
\looseness=-1
Before a call, the runner reserves a charge using declared input/output limits and the recorded price schedule.
We replace the reservation with the confirmed charge.
An unsuccessful response or a response without a confirmed charge ends the task without a retry and keeps the reserved charge.
All $221$ actual requests pass the recorded budget checks.
Of these, $215$ have confirmed charges and six have unconfirmed costs.
Known charges total USD $0.121359$, with an upper bound of USD $0.363432$ after including reservations for calls with unconfirmed costs.

The main reference adds USD $0.08$ per scheduled task and permits $1.25$ times the cumulative reference.
We schedule a separate compilation attempt after the first three agent runs end.
It makes one builder call without repair and costs USD $0.006872$.
The resulting program passes checks on two unseen bindings supplied without model extraction.
We keep the fixed program for serving.
We include this compilation charge in the cumulative cost checked against the main budget and exclude it from the serving-cost comparison.

Four additional tasks use a separate reference of USD $0.02$ per arrival.
Two succeed, one stops on a call interruption, and the first is denied because its USD $0.039834$ binding reservation exceeds the initial USD $0.025$ allowance.
No model call or GUI service occurs for that denied task, although it still increases the cumulative reference by USD $0.02$.
Each recorded cumulative total stays within the stated budget.
These references are given allowances, not measured costs from running every task with the agent, and the denial falls outside the theorem's assumption that default service costs at most its reference charge.

\subsection{Cost bounds with unconfirmed costs}
For each request without a confirmed charge, we keep the known charge as a lower bound and the reservation as an upper bound.
Across all twelve pairs, mean serving costs are bounded by $[8.142,21.590]$ milli-USD for the agent and $[1.202,4.564]$ for the program path.
We do not restrict this comparison to the eight pairs with fully confirmed costs.
For a paired difference, the lower endpoint subtracts the agent's upper cost from the program's lower cost.
The upper endpoint uses the opposite endpoints.
The mean difference is bounded by $[-20.388,-3.577]$ milli-USD.

\looseness=-1
We enumerate all $4^4=256$ ordered bootstrap draws of the four seed clusters, keeping both policies and all three conditions within each selected seed.
Using linear interpolation, we take the $2.5$th percentile of the lower mean differences as the lower endpoint and the $97.5$th percentile of the upper mean differences as the upper endpoint.
This gives a $95\%$ envelope of $[-33.017,4.621]$ milli-USD, which includes zero.
It uses the lower and upper cost bounds in the bootstrap, rather than giving a confidence interval based on exact costs.
The four clusters, one task family, and provider interruptions limit population inference and conclusions about successful GUI service.

\end{document}